%% file: salunet.tex
\documentclass[10pt]{article}
\usepackage[margin=0.82in]{geometry}
\usepackage{amsmath,amssymb,amsthm,mathtools,graphicx}
\usepackage{booktabs,multirow,makecell,microtype,xcolor}
\usepackage[numbers,sort&compress]{natbib}
\usepackage[colorlinks=true,allcolors=blue]{hyperref}
\usepackage{caption,overpic,float}
\usepackage{titlesec}
\titlespacing*{\section}{0pt}{12pt plus 2pt minus 2pt}{6pt}
\titlespacing*{\subsection}{0pt}{9pt plus 2pt minus 2pt}{4pt}
\AtBeginDocument{\setlength{\abovedisplayskip}{7pt plus 2pt minus 2pt}\setlength{\belowdisplayskip}{7pt plus 2pt minus 2pt}}
\usepackage{placeins,etoolbox}
\pretocmd{\section}{\FloatBarrier}{}{}
\pretocmd{\subsection}{\FloatBarrier}{}{}

\newtheorem{theorem}{Theorem}
\newtheorem{proposition}{Proposition}
\newtheorem{lemma}{Lemma}
\newcommand{\E}{\mathbb{E}}

\title{\vspace{-1.3em}\textbf{Geometry, Not Statistics: SaluNet for
Normalization-Free CNNs and Transformers}}
\author{Mourad Zaied\\RTIM Laboratory, University of Gabes, Tunisia}
\date{}

\begin{document}
\maketitle
\vspace{-2.7em}

\begin{center}
\begin{minipage}{\textwidth}
\captionsetup{type=figure}
\centering
\begin{minipage}[t]{0.32\textwidth}
    \vspace{-3pt}\centering
    \includegraphics[width=\linewidth]{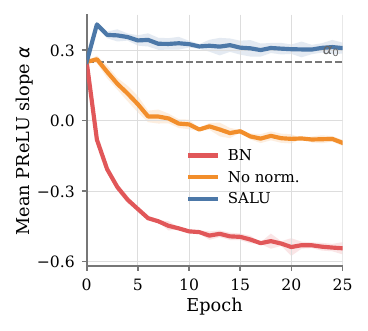}
\end{minipage}\hfill
\begin{minipage}[t]{0.32\textwidth}
    \vspace{0pt}\centering
    \includegraphics[width=\linewidth]{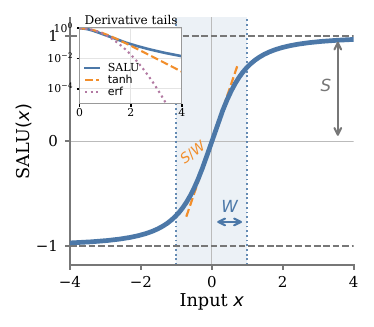}
\end{minipage}\hfill
\begin{minipage}[t]{0.32\textwidth}
    \vspace{0pt}\centering
    \begin{overpic}[width=\linewidth]{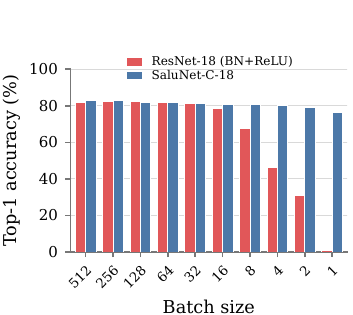}
    \put(50,82.3){\makebox(0,0){\fontsize{6}{7}\selectfont CIFAR-100}}
    \end{overpic}
\end{minipage}
\caption{\textbf{From normalization-dependent activation plasticity to
geometric stabilization.}
\textbf{Left (CIFAR-10):} Mean PReLU slope across three activation layers
under BN, no normalization, and SALU. Curves and bands report the mean and one
standard deviation over three seeds; the dashed line marks $\alpha_0=0.25$.
\textbf{Center:} SALU fixes the geometric center at the origin while learning
the saturation amplitude $S$, transition width $W$, and local gain $S/W$.
The inset contrasts its polynomial derivative decay with those of $\tanh$ and
$\operatorname{erf}$.
\textbf{Right (CIFAR-100):} SaluNet-C-18 remains trainable as batch size
decreases from $512$ to $1$, whereas the matched ResNet-18 (BN+ReLU) baseline
degrades sharply at small batch sizes.}
\label{fig:teaser}
\end{minipage}
\end{center}

\begin{abstract}
Normalization stabilizes deep networks by adapting intermediate
representations through prescribed centering and scaling rules. We investigate
the complementary principle: leave representations in their original
coordinates and adapt the geometry of the pointwise operator instead. Motivated by
controlled evidence that Batch Normalization changes the learning dynamics
of parametric activations, we introduce the Saturated Adaptive Linear Unit
(SALU), a smooth bounded operator with two learned scalars per layer. A
bijective reparameterization exposes its saturation amplitude $S$, transition
width $W$, and central gain $S/W$. These coordinates control the bounded
forward amplitude and distribute backward sensitivity over the input domain.
In a homogeneous mean-field model, forward variance and backward gradient
propagation share the normalized coordinate $v/W^2$, where $v$ denotes the
activation variance, while obeying distinct stability conditions. We reuse the same
primitive to construct the SWALU and GALU adaptive gates and define SaluNet by
replacing BN or LN with SALU and conventional activations with SALU-based
gates. On CIFAR-100, controlled ResNet-18 experiments show that the complete
SALU+SWALU construction outperforms its isolated substitutions. Additional
experiments demonstrate trainability down to batch size one, while
checkpoint-level representation analysis reveals a broader late-stage spectrum
than in the matched BN+ReLU model. ImageNet-1K experiments further demonstrate
scalability to ResNet-50 and ViT-B/16, including ViT-B/16 training with an
effective batch size of 128. Together, these results support learnable
pointwise geometry as a common statistics-free mechanism for signal
stabilization and adaptive gating across convolutional networks and vision
Transformers.
\end{abstract}

\section{Introduction}
\label{sec:introduction}

Normalization layers are central to modern deep learning. Batch Normalization
(BN) uses mini-batch statistics to regulate convolutional representations
\citep{ioffe2015batch}, while Layer Normalization (LN) applies within-example
standardization and is ubiquitous in Transformers
\citep{ba2016layer,vaswani2017attention}. Although effective, statistical
normalization couples signal processing to a prescribed centering and scaling
rule. BN additionally depends on mini-batch composition and may become
unreliable when only a few samples are available.

Our starting point is the controlled experiment summarized in the left panel
of Fig.~\ref{fig:teaser}. We first compare the dynamics of a learnable PReLU
\citep{he2015delving} in the same small CNN with BN and without normalization.
The learned slope follows markedly different trajectories under these two
conditions. Under BN, it undergoes a strong initial displacement followed by
a substantially slower adaptation regime. This observation indicates that
activation plasticity cannot be interpreted independently of the mechanism
controlling its input distribution. For reference, the figure also reports the
trajectory obtained with the pointwise stabilizer proposed later in this
section.

Additional fixed-probe and gradient measurements clarify this interaction.
For PReLU, the gradient does not vanish under BN; rather, its magnitude and
coherence across mini-batches change as the intermediate representation
stabilizes. We observe a complementary coupling with
$\operatorname{PSwish}_{\beta}(x)=x\sigma(\beta x)$. Under BN, $\beta$
increases while the BN scale $\gamma$ decreases, causing the effective gate
sharpness $|\beta\gamma|$ to approach a plateau. These findings, replicated
across three seeds and detailed in Appendix~\ref{app:activation_plasticity},
do not imply that BN prevents activation learning. They show instead that
normalization and learnable activation geometry form a coupled optimization
system whose behavior depends on the activation parameterization.

This coupling suggests viewing stabilization as an alignment problem:
intermediate representations should remain within an operating region where
forward amplitudes are controlled and backward derivatives remain useful.
Statistical normalization addresses this problem by transforming the
representation itself. BN and LN subtract a representation-dependent mean and
normalize its spread, thereby moving and rescaling the distribution into a
prescribed region around zero.

Statistics-free pointwise replacements instead define a bounded sensitive
region through their functional geometry. Dynamic Tanh (DyT) uses an
origin-centered $\tanh$ profile whose learned input scale controls compression
but cannot translate the function when the input distribution is shifted or
asymmetric \citep{zhu2025transformersnormalization}. Boundedness therefore
does not imply a statistically centered or symmetric output distribution.
Derf systematizes this design space through zero-centeredness, boundedness,
center sensitivity, and monotonicity, and adds a learned internal shift
\citep{chen2025strongernormalizationfreetransformers}. The center of its
high-derivative region, $x_c=-s/\alpha$, can consequently move toward a typical
value of the training distribution. Derf translates the function globally for
the layer; unlike BN or LN, it does not recompute a center from each
representation.

SALU adopts a complementary philosophy. It neither moves and rescales the
distribution through activation statistics nor translates the function
through an internal shift. Instead, it adapts the operating geometry presented
to the incoming representation: its bounded amplitude controls forward signal
growth, while its width and central gain determine how much of the input range
is treated quasi-linearly and how much is exposed to stronger nonlinear
shaping. In short, normalization moves the distribution, Derf moves the
function, and SALU deforms the function around a fixed center. We therefore
ask: \emph{can signal stabilization and
activation geometry be represented jointly by the same learnable pointwise
operator?}

Building on this algebraic saturating form, we introduce the Saturated Adaptive
Linear Unit (SALU) as a learnable geometric signal-control operator. Its
contribution lies not in boundedness alone, but in the independent layer-wise
adaptation of saturation amplitude and transition width, its use as a
statistics-free replacement for normalization, and its reuse within the
derived SWALU and GALU gates. SALU exposes a saturation amplitude $S$ and an
input transition width $W$.
The width determines how far the quasi-linear, gradient-sensitive region
extends, $S/W$ controls its local gain, and $S$ sets the limiting output
magnitude. Each layer can therefore adapt its stabilizing geometry without
computing mini-batch, token, or channel statistics.

The implementation learns one positive pair $(a,b)$ per SALU layer. A
bijective reparameterization maps this pair to $(S,W)$ without introducing
additional parameters. We optimize unconstrained log-parameters to learn the positive
coefficients $(a,b)$, and use $(S,W)$ as interpretation and
analysis coordinates. Outside its quasi-linear region, SALU's derivative decreases polynomially rather than exponentially as in $\tanh$ or Gaussianly as in $\operatorname{erf}$. This does not prevent gradients from eventually becoming small, but preserves pointwise sensitivity farther into the saturation regime.

Our contributions are:
\begin{itemize}
\setlength{\itemsep}{1pt}
\item We provide controlled evidence that normalization and learnable
activation geometry form a coupled optimization system through PReLU and
parametric-Swish diagnostics.
\item We introduce SALU as a layer-wise adaptive geometric parameterization of
a bounded, monotone, origin-centered algebraic map. The coordinates $S$, $W$,
and $S/W$ control amplitude, width, and local gain.
\item We derive a unified mean-field description of forward variance and
backward gradients and characterize the forward fixed-point regimes.
\item We derive SWALU and GALU and define SaluNet using the same statistics-free
primitive to replace BN in ResNets and LN in vision Transformers.
\item We evaluate SaluNet through matched CIFAR comparisons, ablations,
CNN batch-size scaling, representation analysis, and ImageNet evaluations across
model families.
\end{itemize}

The reference implementations, experiment configurations, and pretrained checkpoints are available at
\url{https://github.com/MouradZaied/SaluNet/}.

\section{Saturated Adaptive Linear Unit}
\label{sec:salu}

\subsection{Definition and Geometric Coordinates}
\label{subsec:salu_definition}

We define the Saturated Adaptive Linear Unit (SALU) by
\begin{equation}
\phi(x;a,b)=\frac{ax}{\sqrt{1+abx^2}},\qquad a,b>0.
\label{eq:salu_ab}
\end{equation}
Each layer learns one pair $(a,b)$ shared across its channels, spatial
positions, or tokens, adding only two scalars. Positivity is enforced through
$a=\exp(\widetilde a)$ and $b=\exp(\widetilde b)$. The optimizer updates the unconstrained log-parameters $(\tilde{a},\tilde{b})$, while $(a,b)$ denote the positive coefficients appearing in the forward expression. Their roles are coupled, however, so we introduce the bijective geometric coordinates
\begin{equation}
S=\sqrt{\frac ab},\qquad W=\frac1{\sqrt{ab}},
\qquad a=\frac SW,\qquad b=\frac1{SW}.
\label{eq:salu_coordinate_map}
\end{equation}
They introduce no additional parameters and yield
\begin{equation}
\phi(x;S,W)=\frac{Sx}{\sqrt{W^2+x^2}}
=Sg(x/W),\qquad g(u)=\frac{u}{\sqrt{1+u^2}}.
\label{eq:salu_sw}
\end{equation}
Thus $S$ is the saturation amplitude, $W$ is the characteristic transition
width, and $S/W=a$ is the central gain:
\begin{equation}
\lim_{x\to\pm\infty}\phi(x)=\pm S,\qquad
\phi(\pm W)=\pm\frac{S}{\sqrt2},\qquad \phi'(0)=\frac SW.
\end{equation}
The two descriptions are therefore complementary: $(a,b)$ specify the
implemented operator, while $(S,W)$ separate its vertical and horizontal
geometry for interpretation and analysis.

\begin{figure}[h]
\centering

\includegraphics[width=0.88\linewidth]{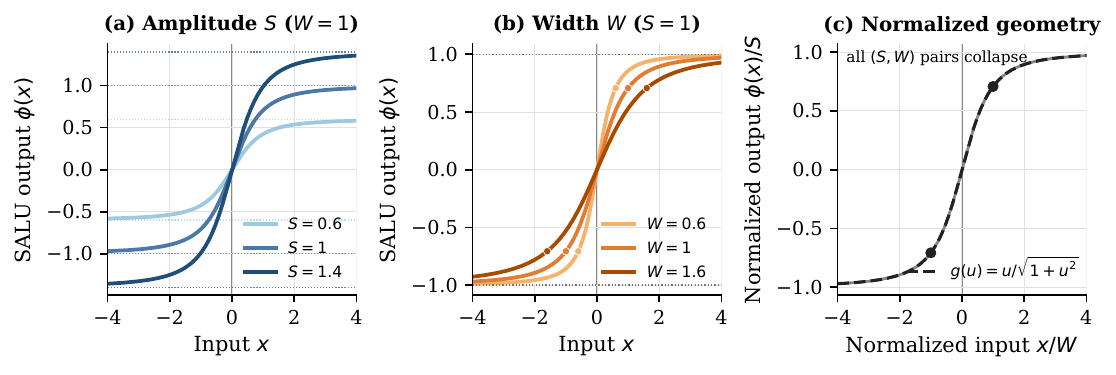}

\caption{\textbf{Geometric parameterization of SALU.}
\textbf{(a)} Varying $S$ at fixed $W=1$ changes the output range; dotted
lines mark $\pm S$. \textbf{(b)} Varying $W$ at fixed $S=1$ changes the
horizontal scale; markers show $\phi(\pm W)=\pm S/\sqrt2$.
\textbf{(c)} Normalizing input by $W$ and output by $S$ collapses all
configurations onto $g(u)=u/\sqrt{1+u^2}$.}
\label{fig:salu_geometric_parameterization}
\end{figure}

\subsection{Geometric and Differential Properties}
\label{subsec:salu_properties}

For every $S,W>0$, SALU is smooth, odd, bounded by $S$, and strictly
increasing; it is convex for $x<0$ and concave for $x>0$. Its derivative is
\begin{equation}
\phi'(x)=\frac{SW^2}{(x^2+W^2)^{3/2}},\qquad
0<\phi'(x)\leq\frac SW=a,
\label{eq:salu_input_derivative_sw}
\end{equation}
so SALU is globally $S/W$-Lipschitz. Its input sensitivity decays as
$\mathcal O(|x|^{-3})$, polynomially rather than exponentially as for $\tanh$.
The two optimized log-parameters retain nonzero individual sensitivities near
saturation, but these leading terms encode the amplitude direction; the
independent width signal still decays as $\mathcal O(|x|^{-2})$. Thus, the
coordinate change does not create information that is absent from the
operator, although it changes how amplitude and width directions are coupled
and scaled for optimization. These are local operator properties, not a claim
of superior conditioning or global optimization stability. Complete
derivations and sensitivities in both coordinate systems are given in
Appendix~\ref{app:salu_properties}.

\section{Signal Propagation Through SALU Networks}
\label{sec:salu_propagation}

We analyze propagation in $(S,W)$ coordinates because they separate amplitude,
width, and central gain; all results translate exactly to $(a,b)$ through
Equation~\eqref{eq:salu_coordinate_map}. The analysis concerns SALU in
infinite-width, bias-free fully connected networks and is not a theorem for
residual CNNs, attention, or the derived gates introduced later.

\subsection{Mean-Field Setting}
\label{subsec:salu_mean_field}

Consider
\begin{equation}
 h_i^{(\ell)}=\sum_j\omega_{ij}^{(\ell)}x_j^{(\ell-1)},\qquad
 x_i^{(\ell)}=\phi_\ell(h_i^{(\ell)}),\qquad
 \omega_{ij}^{(\ell)}\sim
 \mathcal N\!\left(0,\frac{\sigma_{w,\ell}^2}{n_{\ell-1}}\right).
\label{eq:salu_mlp}
\end{equation}
With $\phi_\ell(x)=S_\ell x/\sqrt{W_\ell^2+x^2}$, standard mean-field
assumptions give $h_i^{(\ell)}\Rightarrow\mathcal N(0,v_\ell)$ and
\begin{equation}
v_{\ell+1}=\mathcal V_\ell(v_\ell)=
\sigma_{w,\ell+1}^2S_\ell^2
\E_Z\!\left[\frac{v_\ell Z^2}{W_\ell^2+v_\ell Z^2}\right],
\qquad Z\sim\mathcal N(0,1).
\label{eq:salu_layerwise_variance}
\end{equation}
Hence $0\leq\mathcal V_\ell(v)<\sigma_{w,\ell+1}^2S_\ell^2$.
For fixed-point analysis we set $S_\ell=S$, $W_\ell=W$, and
$\sigma_{w,\ell}=\sigma_w$.

\subsection{Forward Variance and Fixed-Point Regimes}
\label{subsec:salu_forward_fixed_points}

The homogeneous variance map and its slope at the origin are
\begin{equation}
\mathcal V(v)=\sigma_w^2S^2
\E_Z\!\left[\frac{vZ^2}{W^2+vZ^2}\right],\qquad
\chi_0:=\mathcal V'(0)=\sigma_w^2\left(\frac SW\right)^2.
\label{eq:salu_homogeneous_variance}
\end{equation}

\begin{theorem}[Forward-variance regimes]
\label{thm:forward_variance}
The map $\mathcal V$ is strictly increasing, strictly concave, and bounded by
$\sigma_w^2S^2$. If $\chi_0\leq1$, zero is its unique fixed point and attracts
every finite $v_0\geq0$ (at equality, nonlinearly despite marginal linear
stability). If $\chi_0>1$, zero is unstable and a unique $v^\star>0$ emerges;
it satisfies $0<\mathcal V'(v^\star)<1$ and attracts every $v_0>0$.
\end{theorem}

The proof is in Appendix~\ref{app:forward_variance}. Crossing
$S/W=1/\sigma_w$ destabilizes zero variance, while bounded saturation replaces
unbounded growth by a unique positive attractor.

\begin{figure}[h]
    \centering
    \includegraphics[width=0.72\linewidth]
    {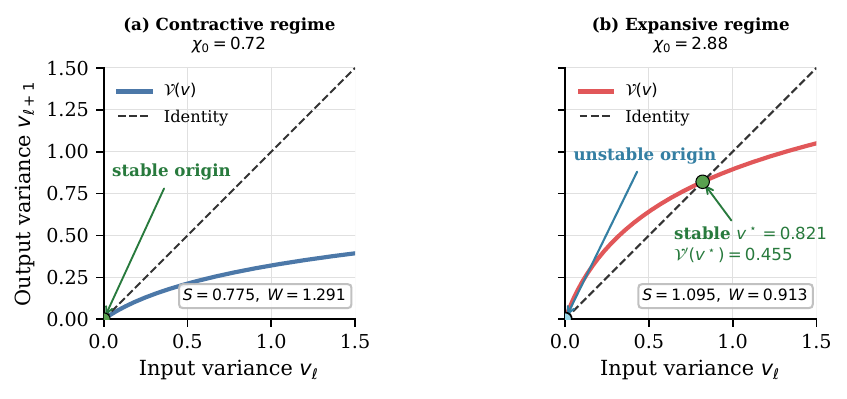}
    \caption{\textbf{Mean-field forward-variance regimes of SALU.}
    Solid curves show $\mathcal V(v)$ and dashed lines show the identity map,
    with $\sigma_w^2=2$.
    \textbf{(a)} $S=\sqrt{0.6}$ and $W=1/\sqrt{0.6}$ give
    $\chi_0=0.72<1$, and the origin is globally attracting.
    \textbf{(b)} $S=\sqrt{1.2}$ and $W=1/\sqrt{1.2}$ give
    $\chi_0=2.88>1$; the origin is unstable and a unique stable positive
    equilibrium $v^\star$ emerges. The parameters only illustrate opposite
    sides of $\chi_0=1$; $v^\star$ is configuration-specific.}
    \label{fig:variance_fixed_points}
\end{figure}

\subsection{Backward Gradient Propagation}
\label{subsec:salu_backward}

Let $q_\ell=\E\|\nabla_{h^{(\ell)}}\mathcal L\|_2^2$ denote the expected squared norm of the pre-activation gradient. Under the same
mean-field independence approximation,
\begin{equation}
q_\ell=\chi(v_\ell)q_{\ell+1},\qquad
\chi(v)=\sigma_w^2\left(\frac SW\right)^2
M\!\left(\frac{v}{W^2}\right),\quad
M(u)=\E_Z[(1+uZ^2)^{-3}].
\label{eq:salu_backward_factor}
\end{equation}
Since $M$ decreases from one to zero, the local edge condition is
\begin{equation}
\chi(0)=1\quad\Longleftrightarrow\quad\sigma_wS/W=1.
\label{eq:salu_local_edge_condition}
\end{equation}
Forward propagation scales with $S^2$, backward propagation with $(S/W)^2$,
and both depend on $v/W^2$: the same geometry controls both recursions,
with different stability criteria. Although $\mathcal V'(0)=\chi(0)$,
the stability of the scalar forward-variance iteration and the criticality
of the gradient second moment differ away from the origin. In particular,
a stable positive forward fixed point does not guarantee critical gradients;
Appendix~\ref{app:backward_propagation} establishes a limitation of the
homogeneous regime.

This analysis connects SALU's geometric coordinates to signal and gradient
dynamics in a mean-field model. It characterizes a positive, bounded forward
equilibrium and distinguishes its stability from backward criticality.
These results provide an interpretation of the operator, not a proof of
stability for the residual and attention architectures evaluated below.

\section{SALU-Based Activation Gates and SaluNet}
\label{sec:salunet}

\subsection{SWALU and GALU}
\label{subsec:salu_gates}

Beyond its direct use as a bounded pointwise stabilizer, SALU can serve as a
learnable odd component inside multiplicative activation gates. We construct
two such operators: SWALU, motivated by Swish-like sigmoid gating
\citep{ramachandran2017searching}, and GALU, derived from the tanh
approximation of GELU \citep{hendrycks2016gaussian}. Each gate instance learns
its own positive SALU parameter pair $(a,b)$.

\paragraph{SWALU.}
The sigmoid identity
\begin{equation}
    \sigma(x)
    =
    \frac{1+\tanh(x/2)}{2}
\end{equation}
expresses a sigmoid gate in terms of a bounded odd function. Replacing this
fixed odd component with SALU gives
\begin{equation}
    \operatorname{SWALU}(x;a,b)
    =
    \frac{x}{2}
    \left[
        1+\phi(x;a,b)
    \right].
\label{eq:swalu_ab}
\end{equation}
SWALU therefore modulates the input using a learned odd geometry whose
saturation amplitude, transition width, and central gain are controlled by
the underlying SALU parameters.

\paragraph{GALU.}
The commonly used tanh approximation of GELU is
\begin{equation}
    \operatorname{GELU}(x)
    \approx
    \frac{x}{2}
    \left[
        1+
        \tanh\!\left(h(x)\right)
    \right],
\end{equation}
where
\begin{equation}
    h(x)
    =
    \sqrt{\frac{2}{\pi}}
    \left(x+0.044715x^3\right).
\label{eq:galu_inner_map}
\end{equation}
Replacing the fixed hyperbolic tangent with SALU yields
\begin{equation}
    \operatorname{GALU}(x;a,b)
    =
    \frac{x}{2}
    \left[
        1+\phi\!\left(h(x);a,b\right)
    \right].
\label{eq:galu_ab}
\end{equation}
GALU retains the nonlinear input transformation of the GELU approximation
while making the geometry of its odd gating component learnable.

\paragraph{Initialization and interpretation.}
For both gates, we initialize
\begin{equation}
    a_0=b_0=1,
    \qquad\text{equivalently}\qquad
    S_0=W_0=1.
\end{equation}
The initial SALU primitive is therefore
\begin{equation}
    \phi_0(x)
    =
    \frac{x}{\sqrt{1+x^2}},
\end{equation}
a smooth, bounded, monotone, and origin-centered approximation to the odd
components appearing in sigmoid- and tanh-based gates. This provides a
familiar initial modulation profile, after which every layer can adapt its
own amplitude and transition width.

At initialization, $\phi_0(x)\in(-1,1)$, so the modulation factor
$(1+\phi_0)/2$ lies in $(0,1)$. This interval is not imposed after
initialization: because $S$ is learnable and may exceed one, the resulting
operator is an adaptive multiplicative modulator rather than a
probability-valued gate. Moreover, SWALU and GALU themselves are not bounded;
boundedness applies to the internal SALU primitive.

SWALU and GALU are alternative SALU-based gates rather than operators assigned
to particular architecture families. Either can be used in a convolutional
network or a vision Transformer. We compare their behavior across
architecture--dataset combinations. Representative geometries are provided in
Appendix~\ref{app:salu_gate_geometry}, and the gate used in each experiment is
reported in Appendix~\ref{app:gate_selection}.

\subsection{The SaluNet Framework}
\label{subsec:salunet_framework}

The mean-field results of Section~\ref{sec:salu_propagation} apply only to the
homogeneous feed-forward model analyzed there. SaluNet extends SALU to
residual CNNs and vision Transformers empirically; we do not claim that these
architectures obey the same scalar variance recursion.

SaluNet modifies a reference architecture through two complementary
substitutions:
\begin{equation}
\begin{aligned}
    \text{statistical normalization}
    &\longrightarrow
    \text{SALU},
    \\
    \text{conventional activation}
    &\longrightarrow
    \text{SWALU or GALU}.
\end{aligned}
\label{eq:salunet_replacement_rule}
\end{equation}
SALU provides bounded pointwise signal control, whereas SWALU or GALU provides
the multiplicative nonlinear modulation. None of these operators computes
statistics over minibatches, channels, spatial positions, or tokens.

Each SALU, SWALU, or GALU instance is controlled by one positive parameter
pair $(a,b)$ shared across all elements processed by that instance. If an
architecture contains $N_s$ SALU stabilizers and $N_g$ SALU-based gates, its
total number of trainable geometric parameters is therefore
\begin{equation}
    2N_s+2N_g
\end{equation}
scalars. Unlike the affine parameters of common normalization layers, this
number does not grow with the number of channels or feature dimensions.

For the SaluNet-C family, normalization replacements are independently
parameterized, whereas activation-gate parameters are shared within each
residual block. The stem has its own gate, and every residual block uses one
SWALU or GALU instance at all of its activation sites. Consequently,
SaluNet-C-18 has $N_s=20$ SALU stabilizers and $N_g=9$ block-wise gates
covering $17$ activation sites, while SaluNet-C-50 has $N_s=53$ and $N_g=17$
covering $49$ activation sites. Their numbers of trainable geometric parameters are therefore
$58$ and $140$ trainable scalars, respectively.

We use \textbf{SaluNet-C} for convolutional architectures and
\textbf{SaluNet-T} for vision Transformers. SaluNet-C-18 and SaluNet-C-50 are
derived from ResNet-18 and ResNet-50, respectively. SaluNet-T-CIFAR denotes
the compact vision Transformer used for CIFAR, while SaluNet-T-B/16 denotes
the ViT-B/16 model evaluated on ImageNet-1K. Apart from the stated
normalization and activation substitutions, the principal structure of each
reference architecture is retained.

SaluNet thus denotes the combination of SALU-based stabilization and adaptive
SALU-based gating, rather than a fixed assignment of SWALU or GALU to a model
family. Detailed replacement rules, gate comparisons, initialization values, and
architecture-specific configurations are provided in
Appendix~\ref{app:architectural_details}.

\section{Experiments}
\label{sec:experiments}

\subsection{Experimental Setup}
\label{subsec:experimental_setup}

We evaluate SaluNet across convolutional and vision Transformer architectures
on CIFAR-10, CIFAR-100, and ImageNet-1K. The convolutional experiments use
SaluNet-C-18 on CIFAR and SaluNet-C-50 on ImageNet-1K, both derived from
residual networks~\citep{he2016deep}. The Transformer
experiments use the compact SaluNet-T-CIFAR architecture on both CIFAR
datasets and SaluNet-T-B/16 on ImageNet-1K. All SaluNet models are trained
from scratch. All experiments use the same shared NVIDIA A40 with 24~GB of available GPU memory.

For ResNet-18/CIFAR-100, all configurations share the backbone, data pipeline,
optimizer, primary learning-rate schedule, and training duration. Learnable
signal-control and gate parameters use the same accelerated cosine schedule.
We report the main model separately from its EMA so that the operator
comparison can be read independently of temporal averaging. EMA decay is
selected per normalization family and reported explicitly. Three independent
runs are used unless a result is marked as single-run. The CIFAR-10
convolutional comparison also uses three independent runs. All compact CIFAR Transformer configurations, including the baselines, are
trained in our own implementation under a common dataset-specific recipe.

On ImageNet-1K, we report Top-1/Top-5 validation accuracy.
Each ImageNet backbone is represented by one completed training run,
with the same selected EMA checkpoint evaluated at both resolutions.
These experiments demonstrate scaling feasibility rather than
quantifying run-to-run variability.
EMA results are reported for both convolutional and Transformer models, including SaluNet-T-CIFAR.
Gate comparisons, replacement rules, preprocessing, optimization,
checkpoint selection, and evaluation conventions are detailed in
Appendices~\ref{app:architectural_details}
and~\ref{app:experimental_details}.

\subsection{CIFAR Classification and Operator Ablations}
\label{subsec:cifar_results}

Table~\ref{tab:cifar_cnn_results} compares the convolutional configurations.
On CIFAR-100, SALU+SWALU reaches $83.02\pm0.29\%$ mean EMA accuracy,
exceeding BN+ReLU by $2.34$ points and BN+SWALU by $1.89$ points.
SALU+ReLU reaches $80.37\pm0.07\%$, supporting the value of the combined
stabilizer and gate in this CNN recipe. GALU yields a close EMA mean of
$82.95\pm0.31\%$ and the highest main-model mean, $82.37\pm0.34\%$.
On CIFAR-10, GALU has the higher observed mean in both columns.
Additional RMSNorm and DyT CNN comparisons are reported in
Table~\ref{tab:app_additional_cnn} (Appendix~\ref{app:cifar_protocols}).

\begin{table}[H]
\centering
\renewcommand{\arraystretch}{0.90}
\caption{\textbf{ResNet-18 results on CIFAR.} EMA and main columns report mean $\pm$ sample standard deviation of per-run best accuracies over three independent runs. Best denotes a separately reported best observed record, not necessarily the maximum of the three runs summarized. On CIFAR-100, BN uses EMA decay $0.997$ and the other stabilizers use $0.9997$. The CIFAR-10 SALU+GALU configuration uses EMA decay $0.999$. For SALU+SWALU/CIFAR-100, Best is $83.27\%$ within the three-run cohort; an additional exploratory record of $83.42\%$ is excluded from these statistics (Appendix~\ref{app:reporting_conventions}).}
\label{tab:cifar_cnn_results}
\small
\begin{tabular}{@{}llcccc@{}}
\toprule
Dataset & Stabilizer + activation & Runs & EMA (\%) & Main (\%) & Best (\%) \\
\midrule
CIFAR-10 & BN + ReLU & 3 & $96.95\pm0.09$ & $96.93\pm0.08$ & $97.02$ \\
 & SALU + SWALU & 3 & $97.08\pm0.05$ & $97.05\pm0.06$ & $97.24$ \\
 & SALU + GALU & 3 & $\mathbf{97.19\pm0.11}$ & $\mathbf{97.17\pm0.10}$ & $\mathbf{97.35}$ \\
\midrule
CIFAR-100 & BN + ReLU & 3 & $80.68\pm0.64$ & $80.89\pm0.25$ & $81.26$ \\
 & BN + SWALU & 3 & $81.13\pm0.24$ & $81.14\pm0.21$ & $81.33$ \\
 & SALU + ReLU & 3 & $80.37\pm0.07$ & $79.97\pm0.19$ & $80.45$ \\
 & GN + ReLU & 3 & $80.75\pm0.28$ & $79.95\pm0.27$ & $80.97$ \\
 & GN + SWALU & 3 & $82.53\pm0.34$ & $81.69\pm0.28$ & $82.78$ \\
 & SALU + GALU & 3 & $82.95\pm0.31$ & $\mathbf{82.37\pm0.34}$ & $83.25$ \\
 & SALU + SWALU & 3 & $\mathbf{83.02\pm0.29}$ & $82.36\pm0.31$ & $\mathbf{83.27}$ \\
\bottomrule
\end{tabular}
\end{table}

Table~\ref{tab:vit_cifar_results} reports matched compact Transformer runs.
All configurations use our seven-block CIFAR ViT with one activation between
the two MLP projections and attention scaled by $1/\sqrt{32}$.
On CIFAR-10, SALU+SWALU improves mean EMA accuracy by $0.52$ points over
SALU+GELU; GALU improves it by $0.31$ points. On CIFAR-100, SALU+GELU
has the highest EMA mean, $70.39\pm0.41\%$, compared with
$70.29\pm0.11\%$ for GALU and $70.11\pm0.48\%$ for SWALU.
SALU+GELU exceeds LN+GELU by $1.73$ points on CIFAR-100, while their
CIFAR-10 EMA means differ by $-0.04$ points.

\begin{table}[H]
\centering
\renewcommand{\arraystretch}{0.90}
\caption{\textbf{Compact ViT results on CIFAR.} All entries are rerun in our implementation for 200 epochs. EMA and main columns report mean $\pm$ sample standard deviation of per-run best accuracies over three independent runs. Best is the maximum across those runs and both model variants. Learning rates are $5\times10^{-4}$ on CIFAR-10 and $10^{-3}$ on CIFAR-100. DyT uses $\alpha_{\rm init}=1$.}
\label{tab:vit_cifar_results}
\small
\begin{tabular}{@{}llcccc@{}}
\toprule
Dataset & Stabilizer + activation & Runs & EMA (\%) & Main (\%) & Best (\%) \\
\midrule
CIFAR-10 & LN + GELU & 3 & $91.43\pm0.31$ & $91.29\pm0.29$ & $91.61$ \\
 & SALU + GELU & 3 & $91.39\pm0.21$ & $91.22\pm0.11$ & $91.63$ \\
 & SALU + GALU & 3 & $91.70\pm0.17$ & $91.47\pm0.15$ & $91.89$ \\
 & SALU + SWALU & 3 & $\mathbf{91.91\pm0.19}$ & $\mathbf{91.92\pm0.16}$ & $\mathbf{92.11}$ \\
\midrule
CIFAR-100 & DyT + GELU & 3 & $65.32\pm0.49$ & $64.15\pm0.70$ & $65.85$ \\
 & RMSNorm + GELU & 3 & $67.86\pm0.36$ & $66.76\pm0.34$ & $68.20$ \\
 & LN + GELU & 3 & $68.66\pm0.13$ & $67.26\pm0.25$ & $68.76$ \\
 & SALU + GALU & 3 & $70.29\pm0.11$ & $\mathbf{69.64\pm0.29}$ & $70.37$ \\
 & SALU + SWALU & 3 & $70.11\pm0.48$ & $69.22\pm0.43$ & $70.56$ \\
 & SALU + GELU & 3 & $\mathbf{70.39\pm0.41}$ & $69.56\pm0.51$ & $\mathbf{70.82}$ \\
\bottomrule
\end{tabular}
\end{table}

These results support context-dependent benefits of SALU and its gates.
The CIFAR-100 Transformer gains are already present with SALU+GELU,
whereas adaptive gates improve the observed CIFAR-10 means. These operator
comparisons do not isolate the benefit of learning gate parameters; matched
learned-versus-fixed gate ablations are needed. Three runs do not establish
statistical equivalence or a universal ordering.

\subsection{Batch-Size Robustness}
\label{subsec:batch_size_robustness}

We next vary the training batch size from $512$ to $1$ while retaining the
same ResNet-18 backbone and dataset-specific training framework. The right
panel of Fig.~\ref{fig:teaser} summarizes the CIFAR-100 results. At batch
size $512$, BN+ReLU and SaluNet-C-18 reach $82.15\%$ and $83.25\%$,
respectively. Their behavior separates as the batch becomes small: at batch
size $8$, accuracy falls to $67.90\%$ with BN+ReLU but remains $80.69\%$
with SaluNet; at batch size $1$, the corresponding values are $1.10\%$ and
$76.23\%$. SaluNet also remains trainable throughout the CIFAR-10 sweep,
retaining $93.44\%$ at batch size $1$.

Below batch size $32$, the two models separate sharply. SaluNet-C-18 degrades
gradually, whereas BN+ReLU undergoes a pronounced collapse, with the gap
increasing from $2.42$ points at batch size $16$ to $75.13$ points at batch
size $1$. These results are consistent with the construction of SALU: its forward map
is applied independently to each activation and does not estimate minibatch
moments. They do not establish that SALU is uniquely robust among
batch-independent operators, nor do they isolate the contributions of SALU
and SWALU. Configurations with batch size at least $8$ were explored across
multiple seeds, whereas batch sizes $4$, $2$, and $1$ use one completed
run per configuration. The table reports retained individual-run
accuracies, not averages across seeds.
We therefore interpret the widening separation as evidence of a
pronounced robustness trend rather than a precise estimate of the
expected accuracy gap. Full values
and protocol details are reported in
Appendix~\ref{app:batch_size_robustness}. Comparisons under a common SWALU
gate with GN and RMSNorm across batch sizes remain future work; the present
sweep does not establish superiority over these methods across batch sizes.

As a complementary feasibility result, SaluNet-T-B/16 is trained on ImageNet-1K
with an effective batch size of $128$, without gradient accumulation
(Section~\ref{subsec:imagenet_scaling}). This single configuration does not
constitute a batch-size sensitivity study for the Transformer.

\subsection{Representation Geometry and Activation Statistics}
\label{subsec:representation_analysis}

We finally compare one trained SaluNet-C-18 checkpoint with one BN+ReLU
counterpart on the complete CIFAR-100 test set. Figure~\ref{fig:representation_geometry}
reports normalized effective rank (NER)~\citep{roy2007effective} and a
normalized participation-ratio index across the four residual stages. The
models are close through the early and middle network, and SaluNet is slightly
lower at Layer~2. The principal separation emerges at Layer~4: NER increases
from $0.1707\pm0.0003$ for BN+ReLU to $0.5183\pm0.0006$ for SaluNet, while
the isotropy index increases from $0.1258\pm0.0004$ to
$0.3203\pm0.0007$. The broader spectrum persists after global average
pooling, where the respective NER values are $0.1615$ and $0.2985$.

\begin{figure}[h]
    \centering
    \includegraphics[width=0.68\linewidth]
    {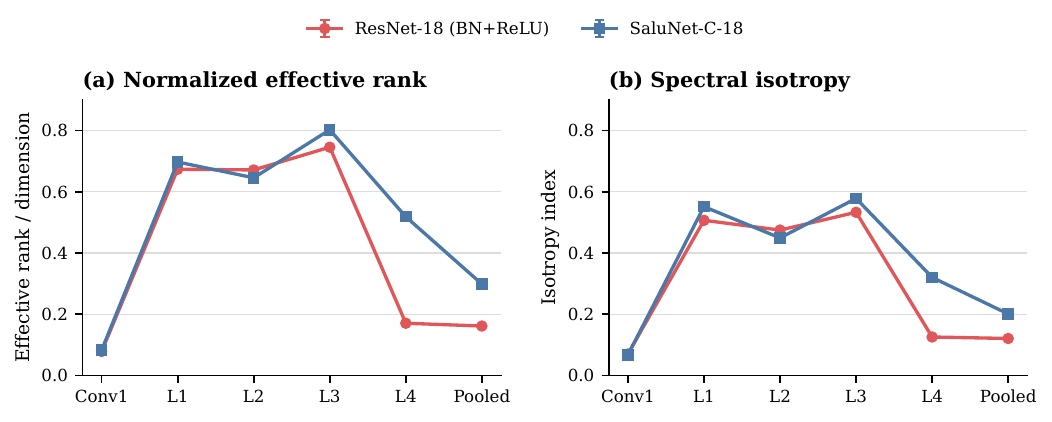}
    \caption{\textbf{Representation-spectrum geometry on CIFAR-100.}
    Normalized effective rank and spectral isotropy across network depth.
    Error bars quantify sensitivity to three deterministic spatial-sampling
    seeds, not variability across independently trained checkpoints. Pooled
    features use all $10{,}000$ test examples.}
    \label{fig:representation_geometry}
\end{figure}

The activation moments further show that SaluNet does not emulate BN
statistics: its intermediate representations can have substantially larger
variance and heavier tails. Nevertheless, the pooled SaluNet representation
is more symmetric and less heavy-tailed in the analyzed checkpoints. Together,
the broader late-stage spectrum and the more regular pooled moments suggest
that SaluNet can preserve a richer collection of feature directions while
delivering a well-behaved representation to the classifier. This offers a
plausible explanation for the accuracy gain, but does not establish that
higher rank or isotropy causes it. Definitions,
sampling details, moment profiles, and numerical values are provided in
Appendix~\ref{app:representation_analysis}. Because only one trained
checkpoint per architecture is compared, the result is descriptive rather
than a causal or architecture-wide claim.

\subsection{Scaling to ImageNet-1K}
\label{subsec:imagenet_scaling}

We evaluate whether the framework remains trainable at ImageNet scale in both
convolutional and Transformer backbones. Table~\ref{tab:imagenet_scaling}
reports independently evaluated EMA checkpoints. SaluNet-C-50 reaches
$79.59\%$ Top-1 and $94.59\%$ Top-5 at $224^2$.
Re-evaluating the same checkpoint at $288^2$, without additional
training or fine-tuning, increases these values to $80.16\%$ and $94.96\%$.
SaluNet-T-B/16 reaches $80.12\%$ Top-1 and $94.79\%$ Top-5 at $224^2$;
re-evaluation at $288^2$ gives $80.28\%$ and $94.74\%$, respectively.
Both ImageNet runs use a single shared NVIDIA A40 GPU with $24$~GB of
memory available to the training process. The ViT-B/16 run adapts the Derf
recipe to physical and effective batch sizes of $128$, with an adjusted
learning rate and no gradient accumulation.

\begin{table}[h]
\centering
\caption{\textbf{Scaling SaluNet to ImageNet-1K.}
Each backbone is represented by one completed training run.
The same selected EMA checkpoint, trained at $224^2$, is evaluated
at both resolutions without additional training or fine-tuning.
Batch values report physical/effective training batch sizes.
These experiments demonstrate scaling feasibility; run-to-run
variability is not quantified.}
\label{tab:imagenet_scaling}
\small
\begin{tabular}{@{}lcccccc@{}}
\toprule
Model & Epochs & Batch & Train & Eval. & Top-1 & Top-5 \\
\midrule
SaluNet-C-50   & 150 & $180/900$ & $224^2$ & $224^2$ & 79.59 & 94.59 \\
SaluNet-C-50   & 150 & $180/900$ & $224^2$ & $288^2$ & 80.16 & 94.96 \\
SaluNet-T-B/16 & 300 & $128/128$ & $224^2$ & $224^2$ & 80.12 & 94.79 \\
SaluNet-T-B/16 & 300 & $128/128$ & $224^2$ & $288^2$ & 80.28 & 94.74 \\
\bottomrule
\end{tabular}
\end{table}

These experiments demonstrate training feasibility at ResNet-50 and ViT-B/16
scale in a memory-constrained, single-GPU setting. Without normalized baselines
trained under matched conditions, they establish neither comparative accuracy
superiority nor savings in memory, training time, or computational cost.
The Transformer run adapts the Derf recipe to a smaller effective batch;
cross-paper differences therefore cannot be attributed to the operator alone.
Full protocols are
given in Appendix~\ref{app:imagenet_protocols}, including the
repeated-augmentation sampling semantics of the C-50 run.

For context, TorchVision ResNet-50 V2 reports $80.858\%$ at $224^2$
~\citep{torchvision2021recipe}; both its 600-epoch run and our 150-epoch C-50
run process approximately $600N$ views, with different sampler semantics
(Appendix~\ref{app:imagenet_protocols}). Published normalization-free results
also exceed ours: timm NF-ResNet-50 reports $80.658\%$ at $288^2$
~\citep{wightman_nfresnet50}, versus our $80.16\%$, while ViT-B reaches
$82.5\%$ with DyT and $82.8\%$ with Derf
~\citep{chen2025strongernormalizationfreetransformers}. These are external
reference points, not matched comparisons: architectures, training resolution,
learning rates, batch sizes, and regularization differ. Neither the C-50 nor
timm checkpoint has a reported across-run uncertainty estimate here;
the $0.498$-point gap establishes neither statistical equivalence nor an
operator-specific disadvantage. Detailed context is given in
Appendix~\ref{app:external_imagenet_context}.

\section{Discussion}
\label{sec:discussion}

\paragraph{Adapting geometry rather than representations.}
The central distinction of SALU is not simply that it avoids minibatch
statistics. BN and LN statistically center and scale each representation so
that typical values occupy a fixed operating region. Derf replaces this
sample-dependent centering by a learned layer-wise internal shift and retains
channel-wise affine output parameters. SALU takes the complementary direction:
it neither statistically recenters the inputs nor learns to translate its
functional center. Instead, $W$ adapts how much of the observed input range
remains in the central sensitive region, $S$ bounds the forward output
amplitude, and $S/W$ controls the local gain. Importantly, $|x|=W$ is not a
hard saturation threshold: beyond it, SALU remains nonlinear and its
derivative decays polynomially. A layer may therefore retain most inputs near
the central region or deliberately expose more of them to nonlinear shaping.
Measurements from a trained SaluNet-C-18 support this layer-specific view;
most stabilizers provide broad central coverage, while a small number operate
substantially farther into the nonlinear regime
(Appendix~\ref{app:salu_operating_regions}). In short, normalization moves the
distribution, Derf moves the function, and SALU deforms the function around a
fixed center. This is the conceptual link between SALU stabilization and the
SWALU/GALU gates, which reuse the same adaptive primitive.

\paragraph{Computational behavior.}
The current operators are expressed using ordinary PyTorch tensor operations,
without a custom CUDA kernel. In eager execution, the resulting sequence of
elementwise operations can be slower than an optimized BN kernel despite
SALU's scalar parameterization. With \texttt{torch.compile}, these operations
can be fused, removing much of the Python and kernel-launch overhead. In
In indicative compiled ResNet-18/CIFAR-100 measurements on a shared
NVIDIA A40 with 24 GB allocated, each epoch, including training and
validation, required approximately 4, 5, and 6 seconds for SALU+ReLU, BN+ReLU, and SALU+SWALU, respectively. These measurements used the same batch size, compilation settings, and validation protocol, during epochs 150--200, after compilation had completed. They suggest that the adaptive gate contributes more to the observed runtime than the SALU stabilizer itself. These indicative timings depend on the implementation and hardware and are not presented as a universal speed benchmark.

\paragraph{Scope.}
SALU does not reproduce zero empirical mean or unit empirical variance, and
our MLP mean-field theorem is not a propagation theorem for residual CNNs or
Transformers. The large-scale results show that the construction remains
optimizable in these architectures, while the batch-size and representation
experiments provide empirical evidence for its distinct behavior. Establishing
architecture-specific propagation theory and matched systems benchmarks
remains future work.

\section{Related Work}
\label{sec:related_work}

\paragraph{Statistical normalization.}
BN regulates channel activations using minibatch moments and has become a
standard component of convolutional networks~\citep{ioffe2015batch}, whereas
LN performs within-example standardization and underlies most Transformer
architectures~\citep{ba2016layer,vaswani2017attention}. GN and RMSNorm reduce
or remove batch dependence by changing which dimensions are aggregated or by
omitting explicit mean subtraction~\citep{wu2018group,zhang2019root}. These
methods differ in their statistics, but all prescribe a representation-level
scaling rule. SALU instead applies the same learned scalar geometry
independently to every element and computes no activation moments.

\paragraph{Networks without normalization.}
Several approaches stabilize deep models without explicit normalization.
Self-normalizing networks exploit a fixed activation and initialization
regime~\citep{klambauer2017self}; Fixup rescales residual branches at
initialization~\citep{zhang2019fixup}; and ReZero learns residual-branch
gates initialized at zero~\citep{bachlechner2021rezero}. Normalizer-Free
ResNets use scaled weight standardization and residual scaling; NFNets add
adaptive gradient clipping and architectural changes for large-scale vision
~\citep{brock2021characterizing,brock2021high}. SaluNet is complementary:
it introduces a learned pointwise stabilizer and derives adaptive gates from
the same primitive, while retaining the principal structure of the reference
CNN or Transformer. NFNet-F0 reaches $83.6\%$ with 71.5 million parameters,
using a substantially larger backbone~\citep{brock2021high}.

\paragraph{Pointwise replacements and adaptive activations.}
FRN combines per-channel normalization with a thresholded activation, while
EvoNorm searches normalization and activation jointly as a tensor-to-tensor
operator~\citep{singh2020filter,liu2020evolving}. These methods already show
the value of treating signal control and nonlinearity as a coupled design
problem, although their operators may still compute feature statistics.
DyT showed that a scaled hyperbolic tangent can replace normalization layers
in Transformers~\citep{zhu2025transformersnormalization}. Derf subsequently
identified favorable properties of pointwise replacements and introduced a
learned shifted error-function map with channel-wise affine output
parameters~\citep{chen2025strongernormalizationfreetransformers}. SALU follows
this statistics-free direction but makes no claim of distributional
centering: odd symmetry yields centered outputs only for correspondingly
centered symmetric inputs. It instead keeps its geometric center at the origin
and learns explicit saturation and transition scales; its input sensitivity
also decays polynomially rather than exponentially or Gaussianly. More
broadly, PReLU, Swish, and GELU demonstrate the importance of
activation shape and gating~\citep{he2015delving,ramachandran2017searching,
hendrycks2016gaussian}. Our diagnostic experiments emphasize that such
learnable activation geometry co-adapts with the preceding signal-control
operator, motivating the unified SALU/SWALU/GALU construction.

\section{Conclusion}
\label{sec:conclusion}

SaluNet combines bounded SALU stabilization with SWALU/GALU activation
gates through one learned pointwise geometry. Saturation amplitude,
transition width, and central gain adapt without computing activation moments
or translating the functional center. The construction remains trainable in
residual CNNs and vision Transformers from CIFAR to ImageNet-1K.
Controlled CNN comparisons support the combined operators, while matched
CIFAR Transformer results show that the additional value of adaptive gates is
dataset-dependent. A learned-versus-fixed gate ablation remains necessary to
isolate the contribution of gate-parameter learning.

The batch-size study shows robustness down to one example per update, and
the analyzed CNN checkpoint has a broader late-stage representation spectrum.
Our propagation analysis characterizes homogeneous SALU MLPs; extending it to
residual and attention architectures, repeating large-scale matched
comparisons, and benchmarking optimized implementations remain future work.

\paragraph{Acknowledgments.}
The author thanks Google Cloud for research credits through the Google Cloud Research Credits program and the University of Gabes for institutional support.

\clearpage
\bibliographystyle{unsrtnat}
\bibliography{salunet_references}

\clearpage
\appendix
\input{appendix_activation_plasticity}
\input{appendix_salu_properties}

\input{appendix_forward_variance}
\input{appendix_backward_propagation}
\input{appendix_salu_gate_geometry}
\input{appendix_architectural_details}
\input{appendix_salu_learning_analysis}
\input{appendix_representation_analysis}
\input{appendix_experimental_details}
\input{appendix_final_geometry_values}

\clearpage
\input{appendix_external_imagenet_context}

\end{document}

%% file: appendix_activation_plasticity.tex
\section{Normalization-Dependent Activation Plasticity}
\label{app:activation_plasticity}

\subsection{Experimental Setup}
\label{app:plasticity_setup}

We study how the mechanism preceding a learnable activation affects its
parameter and gradient dynamics. All experiments use CIFAR-10 and the same
four-layer CNN,
\begin{equation}
\operatorname{Conv}(32)-\operatorname{Conv}(64)-
\operatorname{FC}(128)-\operatorname{FC}(10),
\end{equation}
with $2\times2$ max pooling after each convolutional activation. Convolutional
layers and the first fully connected layer use no bias.

We compare BN before each activation, SALU replacing each corresponding BN,
and no intermediate normalization. BN uses standard channel-wise affine
parameters. In these diagnostics only, each SALU learns one positive pair
$(a,b)$ through log-parameters initialized at $(1,1)$, corresponding to
$(S,W)=(1,1)$. Benchmark architectures use their separately reported
initializations, typically $(a,b)=(1,0.1)$ with architecture-specific
exceptions. Both choices have unit initial local gain $S/W=a=1$ but different
initial amplitudes and widths.

PReLU uses one scalar $\alpha$ per activation layer initialized at
$\alpha_0=0.25$. The replication uses
\begin{equation}
\operatorname{PSwish}_{\beta}(x)=x\sigma(\beta x),\qquad \beta_0=1.
\end{equation}
All models train for 25 epochs with batch size 128, SGD learning rate $0.01$,
momentum $0.9$, Nesterov acceleration, and weight decay $5\times10^{-4}$.
The learning rate is constant; activation and SALU parameters use zero weight
decay. We use random crops with four-pixel padding, horizontal flips, and
CIFAR-10 channel standardization. Results use three independent runs.

Activation trajectories first average the three scalar activation parameters
within each model and then aggregate across seeds. Layer-wise trajectories
remain available in the released logs; layers are not treated as independent
runs.

\subsection{PReLU Parameter Trajectories}
\label{app:prelu_trajectories}

Starting from $\alpha_0=0.25$, the conditions select distinct trajectories.
At epoch 25, the mean slope is $-0.544\pm0.026$ under BN,
$-0.095\pm0.013$ without normalization, and $0.309\pm0.023$ after SALU.
These values are not independent measures of activation quality because each
PReLU receives a different input distribution. They demonstrate instead that
the preceding signal-control mechanism changes the optimization trajectory
and functional regime selected by the same activation.

For a parameter or state vector $\boldsymbol\theta$, define
\begin{equation}
\Delta_{\boldsymbol\theta}^{(t)}=
\frac{\|\boldsymbol\theta^{(t)}-\boldsymbol\theta^{(t-1)}\|_2}
{\|\boldsymbol\theta^{(t-1)}\|_2+\varepsilon}.
\label{eq:relative_parameter_change}
\end{equation}
In the BN condition we separately construct vectors from the PReLU slopes,
the affine parameters $(\boldsymbol\gamma,\boldsymbol\beta)$, and the running
statistics $(\boldsymbol\mu,\boldsymbol v)$.

Figure~\ref{fig:app_prelu_dynamics}(b) shows that the largest PReLU changes
occur during the initial BN transition and then decrease substantially. BN
affine parameters change more slowly, while running statistics continue to
evolve. This association does not establish that BN freezes PReLU learning;
it motivates the fixed-probe and gradient measurements below. At epoch 25,
the diagnostic models reach mean test accuracies of $80.76\%$ for BN,
$78.42\%$ for SALU, and $77.28\%$ without normalization. This experiment is
intended to isolate parameter dynamics, not establish an accuracy advantage.

\begin{figure}[h]
\centering
\includegraphics[width=0.82\linewidth]{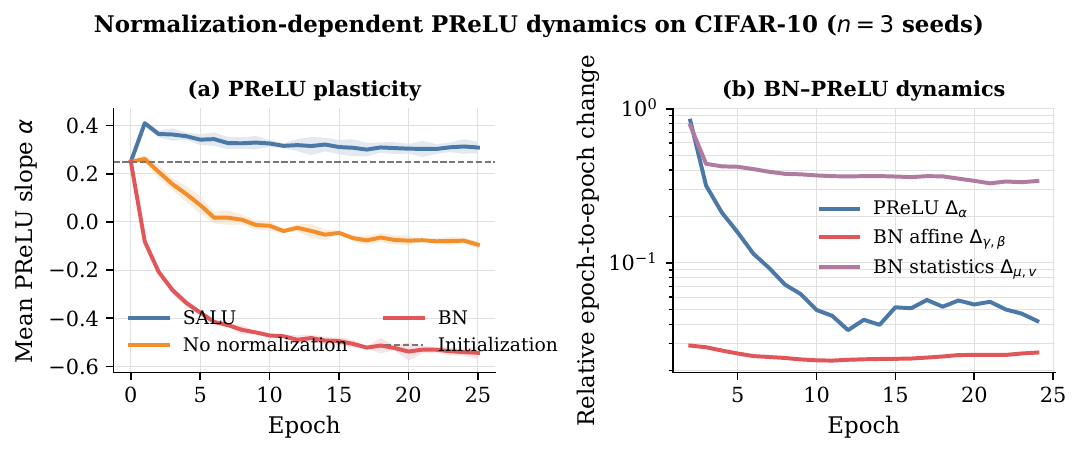}
\caption{\textbf{Normalization-dependent PReLU dynamics on CIFAR-10.}
\textbf{Left:} Mean PReLU slope under BN, SALU, and no normalization.
\textbf{Right:} Relative changes of PReLU slopes, BN affine parameters, and
BN running statistics. Curves and bands report the mean and one standard
deviation over three seeds; the dashed line marks $\alpha_0=0.25$.}
\label{fig:app_prelu_dynamics}
\end{figure}

\subsection{Fixed-Probe and Gradient Dynamics}
\label{app:prelu_mechanism}

\paragraph{Fixed-probe dynamics.}
We use the same 256 deterministic CIFAR-10 test images after every epoch and
record the outputs immediately before the three BN modules. Per-layer and
per-channel means and variances are concatenated into a moment vector
$\mathbf m^{(t)}$ and evaluated with
Eq.~\eqref{eq:relative_parameter_change}. This is a diagnostic forward pass, not conventional evaluation: BN stays in training mode to use the fixed probe batch statistics. Even without gradient computation, this mode updates running-statistic buffers. Saving them before the probe and restoring them afterward prevents the diagnostic from altering subsequent training or evaluation.

During epochs 1--3, the relative changes average $1.462$ for the fixed-probe
moments and $0.844$ for the PReLU vector. During epochs 15--25, they decrease
to $0.340$ and $0.049$. PReLU therefore enters a slower regime after the
initial representation transition, although the probe moments do not become
stationary. This temporal relation does not establish causal direction.

\begin{figure}[h]
\centering
\includegraphics[width=0.82\linewidth]{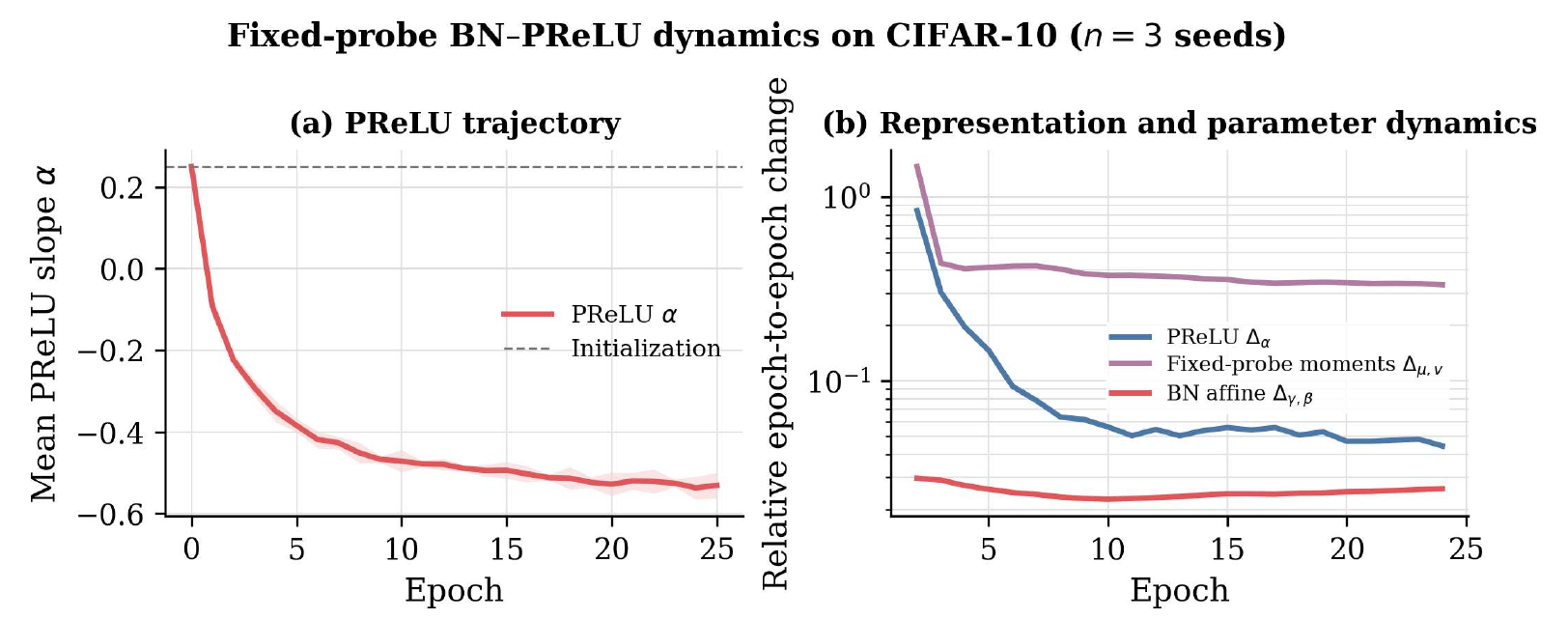}
\caption{\textbf{Fixed-probe BN--PReLU dynamics on CIFAR-10.}
\textbf{Left:} Mean PReLU trajectory under BN.
\textbf{Right:} Relative changes of PReLU, pre-BN fixed-probe moments, and BN
affine parameters. Curves and bands report the mean and one standard deviation
over three seeds.}
\label{fig:app_fixed_probe}
\end{figure}

\paragraph{Direct gradient measurements.}
Let $g_{\ell,k}$ be the scalar gradient of the PReLU parameter in layer
$\ell$ on minibatch $k$. We report
\begin{equation}
G_{\mathrm{RMS}}=
\left(\frac1L\sum_{\ell=1}^L\frac1K\sum_{k=1}^K g_{\ell,k}^2\right)^{1/2}
\end{equation}
and the directional agreement
\begin{equation}
C=\frac1L\sum_{\ell=1}^L
\frac{|\sum_{k=1}^K g_{\ell,k}|}{\sum_{k=1}^K|g_{\ell,k}|+\varepsilon}.
\label{eq:prelu_gradient_coherence}
\end{equation}
During epochs 15--25, gradient RMS is $0.148\pm0.003$ under BN,
$0.305\pm0.009$ under SALU, and $0.277\pm0.015$ without normalization.
Coherence is low in every condition: $0.0070$, $0.0055$, and $0.0069$,
respectively. BN therefore produces nonzero but lower-magnitude late-stage
gradients; directional cancellation is not unique to BN and cannot alone
explain the trajectories.

\begin{figure}[h]
\centering
\includegraphics[width=0.82\linewidth]{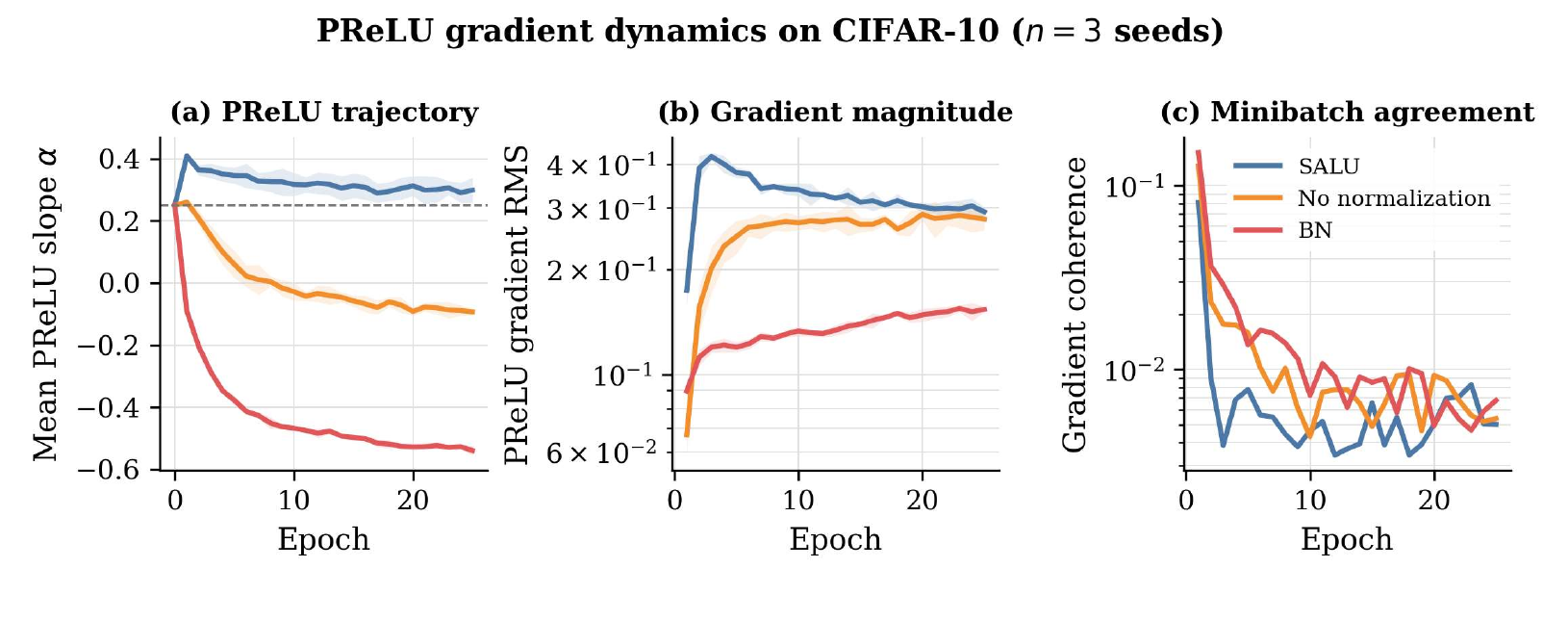}
\caption{\textbf{PReLU gradient dynamics on CIFAR-10.}
\textbf{Left:} Mean PReLU trajectory. \textbf{Center:} Gradient RMS.
\textbf{Right:} Gradient coherence. Curves report means over three seeds.
Bands denote one standard deviation for the trajectory and RMS panels and are
omitted from the coherence panel for readability.}
\label{fig:app_prelu_gradients}
\end{figure}

\input{appendix_pswish}

\paragraph{Takeaway.}
These diagnostics do not establish that normalization universally reduces
activation plasticity. They show that a learnable activation cannot generally
be analyzed independently of its preceding signal-control mechanism. SALU
motivates an alternative in which stabilization and nonlinear geometry are
represented explicitly within learnable pointwise operators without
mini-batch statistics.

%% file: appendix_pswish.tex
\subsection{Replication with Parametric Swish}
\label{app:pswish_dynamics}

To determine whether the preceding observations are specific to the
piecewise-linear structure of PReLU, we repeat the experiment with a smooth
parametric Swish activation,
\begin{equation}
    \operatorname{PSwish}_{\beta}(x)
    = x\,\sigma(\beta x),
    \qquad \beta_0=1,
\end{equation}
using one learned scalar $\beta$ per activation layer.  We retain the same
CNN, CIFAR-10 protocol, optimization settings, and three random seeds described
above.  We compare PSwish preceded by BN, preceded by SALU, and used without an
intermediate signal-control operator.

The interpretation of $\beta$ changes when PSwish follows BN.  Ignoring the
numerical constant in the BN denominator, its input can be written as
\begin{equation}
    y = \gamma \widehat{x}+\delta,
\end{equation}
where $\widehat{x}$ is the standardized activation and $(\gamma,\delta)$ are
the learned channel-wise BN affine parameters.  The sigmoid gate therefore
depends on
\begin{equation}
    \beta y=(\beta\gamma)\widehat{x}+\beta\delta.
\end{equation}
Thus, $|\beta\gamma|$ controls the scale-dependent sharpness of the gate,
whereas $\beta\delta$ shifts its transition.  The raw activation parameter
$\beta$ consequently does not determine the learned gate geometry by itself.

Figure~\ref{fig:pswish_dynamics} makes this coupling explicit.  Under BN, the
mean raw parameter increases from $1$ to $3.236\pm0.022$ by epoch 25.  Over the
same period, the mean effective sharpness reaches only
$2.465\pm0.030$, because the learned BN scales decrease in the first two
activation layers.  The separation between the two trajectories grows
progressively: the ratio of their means decreases from $1.00$ at
initialization to $0.76$ at epoch 25.  SALU also introduces scale coupling. Writing $\phi(x)=Sg(x/W)$ with $g(t)=t/\sqrt{1+t^2}$ gives
\begin{equation}
\operatorname{PSwish}_{\beta}(\phi(x))=Sg(x/W)\,\sigma\!\left(\beta Sg(x/W)\right).
\end{equation}
The gate thus depends on $\beta S$ and the input coordinate $x/W$, rather than on $\beta$ alone. Unlike BN, SALU has no channel-wise affine scale or shift, but its learned scalar amplitude still couples to the activation parameter. Without either operator, interpreting $\beta$ likewise requires the input scale.

The gradient measurements rule out a vanishing-gradient explanation.  During
epochs 15--25, the mean PSwish gradient RMS is
$0.00475\pm0.00008$ under BN, compared with
$0.0742\pm0.0021$ after SALU and $0.1063\pm0.0057$ without normalization.
Although smaller, the BN gradient remains nonzero and exhibits greater
minibatch coherence in this experiment.  BN therefore does not prevent
$\beta$ from learning.  Rather, it distributes the effective gate geometry
across the activation parameter and the preceding affine transformation.
Together with the PReLU results, this experiment shows that
normalization-dependent activation dynamics persist across qualitatively
different parameterizations, while their precise form depends on how each
activation parameter interacts with input scale.

\begin{figure}[h]
    \centering
    \includegraphics[width=0.72\linewidth]
    {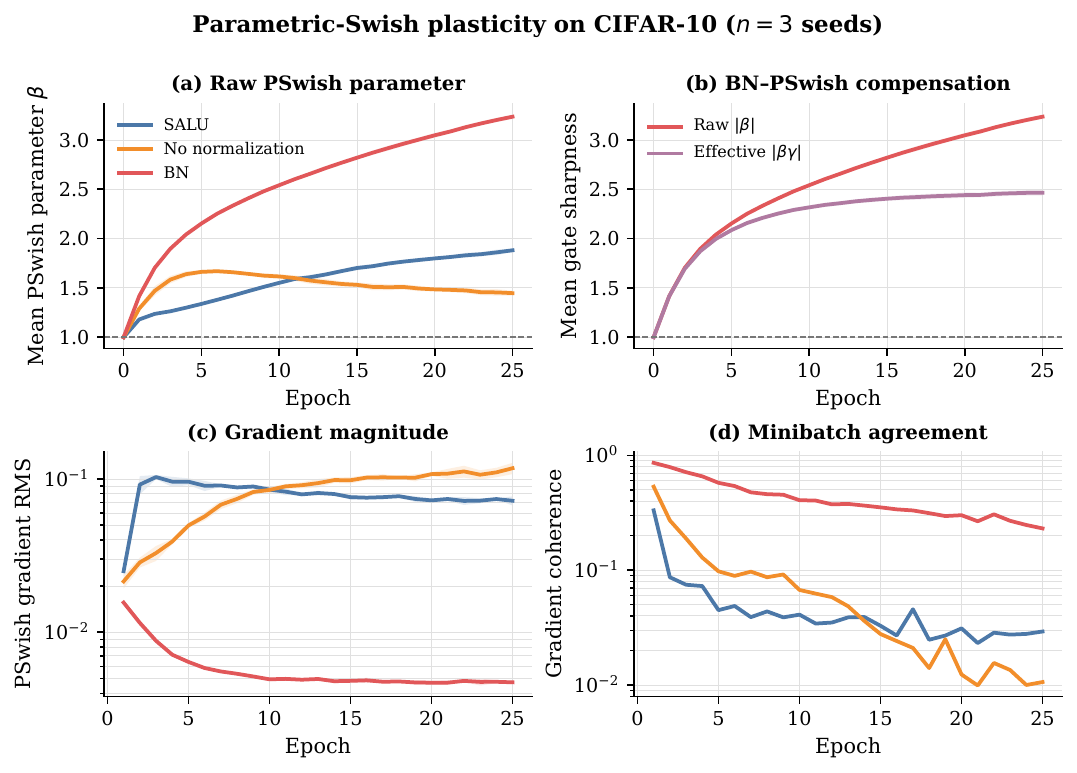}
    \caption{\textbf{Parametric-Swish dynamics on CIFAR-10.}
    \textbf{(a)} Mean learned PSwish parameter $\beta$ under SALU, no
    normalization, and BN; the dashed line marks initialization.
    \textbf{(b)} Under BN, the raw parameter $|\beta|$ continues to increase,
    whereas the scale-dependent gate sharpness $|\beta\gamma|$ approaches a
    slower regime because of the learned BN scale.
    \textbf{(c)} RMS gradient of $\beta$.
    \textbf{(d)} minibatch gradient coherence. Curves first average the three
    activation layers within each run and then report the mean over three
    seeds; shaded regions denote one standard deviation across seeds.}
    \label{fig:pswish_dynamics}
\end{figure}

%% file: appendix_salu_properties.tex
\section{Detailed Differential Properties of SALU}
\label{app:salu_properties}

This appendix derives the pointwise properties summarized in
Section~\ref{subsec:salu_properties}. We distinguish the geometric coordinates
$(S,W)$, the effective implementation parameters $(a,b)$, and the
unconstrained log-parameters $(\widetilde a,\widetilde b)$ updated by the
optimizer.

\subsection{Pointwise Geometric Properties}
\label{app:salu_pointwise_properties}

Consider
\begin{equation}
    \phi(x;S,W)
    =
    \frac{Sx}{\sqrt{x^2+W^2}},
    \qquad S,W>0.
\label{eq:app_salu_sw}
\end{equation}
Oddness and boundedness follow directly:
\begin{equation}
    \phi(-x;S,W)=-\phi(x;S,W),
    \qquad
    \phi(0;S,W)=0,
\end{equation}
and
\begin{equation}
    \lim_{x\rightarrow\pm\infty}\phi(x;S,W)=\pm S.
\end{equation}
Therefore,
\begin{equation}
    |\phi(x;S,W)|<S
\end{equation}
for every finite $x$.

Differentiation with respect to the input gives
\begin{align}
    \frac{\partial\phi}{\partial x}
    &=
    S(x^2+W^2)^{-1/2}
    -
    Sx^2(x^2+W^2)^{-3/2}
    \\
    &=
    \frac{SW^2}{(x^2+W^2)^{3/2}}.
\label{eq:app_salu_first_derivative}
\end{align}
Since $S,W>0$, the derivative is strictly positive for every finite $x$.
SALU is therefore strictly increasing.

Its second derivative is
\begin{equation}
    \frac{\partial^2\phi}{\partial x^2}
    =
    -\frac{3SW^2x}{(x^2+W^2)^{5/2}}.
\label{eq:app_salu_second_derivative}
\end{equation}
Consequently,
\begin{equation}
    \frac{\partial^2\phi}{\partial x^2}>0
    \quad\text{for }x<0,
    \qquad
    \frac{\partial^2\phi}{\partial x^2}<0
    \quad\text{for }x>0.
\end{equation}
SALU is therefore convex on the negative half-line and concave on the positive
half-line.

The first derivative is maximal at the origin:
\begin{equation}
    \sup_{x\in\mathbb R}
    \left|
        \frac{\partial\phi}{\partial x}
    \right|
    =
    \left.
        \frac{\partial\phi}{\partial x}
    \right|_{x=0}
    =
    \frac{S}{W}.
\label{eq:app_salu_lipschitz}
\end{equation}
The mean-value theorem therefore establishes that SALU is globally
$S/W$-Lipschitz.

\subsection{Geometric Parameter Sensitivities}
\label{app:salu_geometric_sensitivities}

Differentiating with respect to the geometric coordinates gives
\begin{equation}
    \frac{\partial\phi}{\partial S}
    =
    \frac{x}{\sqrt{x^2+W^2}},
\label{eq:app_salu_derivative_s}
\end{equation}
and
\begin{equation}
    \frac{\partial\phi}{\partial W}
    =
    -\frac{SxW}{(x^2+W^2)^{3/2}}.
\label{eq:app_salu_derivative_w}
\end{equation}
Their asymptotic behavior is
\begin{align}
    \frac{\partial\phi}{\partial x}
    &\sim
    \frac{SW^2}{|x|^3},
    \label{eq:app_salu_asymptotic_x}
    \\
    \frac{\partial\phi}{\partial S}
    &\longrightarrow
    \operatorname{sign}(x),
    \label{eq:app_salu_asymptotic_s}
    \\
    \frac{\partial\phi}{\partial W}
    &\sim
    -\frac{SW\operatorname{sign}(x)}{|x|^2}.
\label{eq:app_salu_asymptotic_w}
\end{align}

Thus, the input Jacobian and width sensitivity decay polynomially as the
operator approaches saturation. In contrast, the direct sensitivity to the
saturation amplitude remains nonzero.

The width derivative additionally satisfies
\begin{equation}
    \left.
        \frac{\partial\phi}{\partial W}
    \right|_{x=0}
    =0,
    \qquad
    \frac{\partial\phi}{\partial W}
    \longrightarrow0
    \quad\text{as }|x|\rightarrow\infty.
\end{equation}
Hence, $W$ is influenced primarily by inputs lying between the central and
strongly saturated regimes, whereas large-magnitude inputs can continue to
adjust the output amplitude through $S$.

These statements concern local derivatives of the SALU operator. Complete
loss gradients additionally contain the upstream derivative propagated
through the remainder of the network.

\subsection{Implementation-Coordinate Derivatives}
\label{app:salu_implementation_sensitivities}

The implemented idealized operator is
\begin{equation}
    y
    =
    \phi(x;a,b)
    =
    ax(1+abx^2)^{-1/2}.
\label{eq:app_salu_ab}
\end{equation}
Differentiation with respect to the input yields
\begin{align}
    \frac{\partial y}{\partial x}
    &=
    a(1+abx^2)^{-1/2}
    -
    a^2bx^2(1+abx^2)^{-3/2}
    \\
    &=
    \frac{a}{(1+abx^2)^{3/2}}.
\label{eq:app_salu_derivative_x_ab}
\end{align}
Because
\begin{equation}
    \frac{y}{x}
    =
    \frac{a}{\sqrt{1+abx^2}},
    \qquad x\neq0,
\end{equation}
the same derivative can be written as
\begin{equation}
    \frac{\partial y}{\partial x}
    =
    \frac{1}{a^2}
    \left(\frac{y}{x}\right)^3.
\label{eq:app_salu_output_derivative_x}
\end{equation}
At $x=0$, this expression is understood through its continuous extension:
\begin{equation}
    \left.
        \frac{\partial y}{\partial x}
    \right|_{x=0}
    =a.
\end{equation}

Differentiating with respect to $a$ gives
\begin{align}
    \frac{\partial y}{\partial a}
    &=
    x(1+abx^2)^{-1/2}
    -
    \frac{abx^3}{2}(1+abx^2)^{-3/2}
    \\
    &=
    \frac{x\left(1+\frac12abx^2\right)}
         {(1+abx^2)^{3/2}}.
\label{eq:app_salu_derivative_a_direct}
\end{align}
This admits the output-based representation
\begin{equation}
    \frac{\partial y}{\partial a}
    =
    \frac{y}{a}
    -
    \frac{by^3}{2a^2}.
\label{eq:app_salu_derivative_a_output}
\end{equation}
Similarly,
\begin{align}
    \frac{\partial y}{\partial b}
    &=
    -\frac{a^2x^3}{2(1+abx^2)^{3/2}}
    \\
    &=
    -\frac{y^3}{2a}.
\label{eq:app_salu_derivative_b_output}
\end{align}

The bounds derived geometrically can also be recovered directly:
\begin{equation}
    y^2
    =
    \frac{a^2x^2}{1+abx^2}
    <
    \frac{a}{b},
\end{equation}
and
\begin{equation}
    0<
    \frac{\partial y}{\partial x}
    =
    \frac{a}{(1+abx^2)^{3/2}}
    \leq a.
\end{equation}
Therefore,
\begin{equation}
    |y|<\sqrt{\frac{a}{b}},
    \qquad
    \sup_x
    \left|
        \frac{\partial y}{\partial x}
    \right|
    =a.
\end{equation}
These results are equivalent to the geometric bounds
$|\phi|<S$ and
$\sup_x|\phi'(x)|=S/W$.

\subsection{Gradients of the Optimized Log-Parameters}
\label{app:salu_log_parameter_gradients}

Positivity is enforced through
\begin{equation}
    a=\exp(\widetilde a),
    \qquad
    b=\exp(\widetilde b).
\label{eq:app_salu_log_parameterization}
\end{equation}
The chain rule gives
\begin{align}
    \frac{\partial y}{\partial\widetilde a}
    &=
    a\frac{\partial y}{\partial a}
    =
    y-\frac{by^3}{2a},
\label{eq:app_salu_derivative_loga}
    \\
    \frac{\partial y}{\partial\widetilde b}
    &=
    b\frac{\partial y}{\partial b}
    =
    -\frac{by^3}{2a}.
\label{eq:app_salu_derivative_logb}
\end{align}
These are the local derivatives received by the variables updated by the
optimizer.

The relation between the optimized and geometric coordinates is
\begin{equation}
    \log S
    =
    \frac{\widetilde a-\widetilde b}{2},
    \qquad
    \log W
    =
    -\frac{\widetilde a+\widetilde b}{2}.
\label{eq:app_salu_log_coordinate_relation}
\end{equation}
Consequently,
\begin{align}
    \frac{\partial\phi}{\partial\widetilde a}
    &=
    \frac{S}{2}
    \frac{\partial\phi}{\partial S}
    -
    \frac{W}{2}
    \frac{\partial\phi}{\partial W},
\label{eq:app_salu_chain_loga}
    \\
    \frac{\partial\phi}{\partial\widetilde b}
    &=
    -\frac{S}{2}
    \frac{\partial\phi}{\partial S}
    -
    \frac{W}{2}
    \frac{\partial\phi}{\partial W}.
\label{eq:app_salu_chain_logb}
\end{align}
Substitution of
Equations~\eqref{eq:app_salu_derivative_s} and
\eqref{eq:app_salu_derivative_w} recovers exactly
Equations~\eqref{eq:app_salu_derivative_loga} and
\eqref{eq:app_salu_derivative_logb}.

The amplitude and width directions can be isolated exactly by taking the
difference and sum of the two optimized-coordinate derivatives:
\begin{align}
    \frac{\partial\phi}{\partial\widetilde a}
    -
    \frac{\partial\phi}{\partial\widetilde b}
    &=
    S\frac{\partial\phi}{\partial S},
\label{eq:app_salu_amplitude_direction}
    \\
    \frac{\partial\phi}{\partial\widetilde a}
    +
    \frac{\partial\phi}{\partial\widetilde b}
    &=
    -W\frac{\partial\phi}{\partial W}.
\label{eq:app_salu_width_direction}
\end{align}
Hence, the difference carries the amplitude information, whereas the sum
carries the width information. The nonvanishing of the two individual
derivatives must therefore not be interpreted as two independent
nonvanishing geometric signals.

\subsection{Differential Behavior Near Saturation}
\label{app:salu_saturation_behavior}

As $|x|\rightarrow\infty$,
\begin{equation}
    y
    \longrightarrow
    S\operatorname{sign}(x).
\end{equation}
At the same time,
\begin{equation}
    \frac{\partial y}{\partial x}
    \longrightarrow0.
\end{equation}
Thus, once the operator approaches saturation, additional changes in the input
have progressively less influence on its output.

The optimized parameter derivatives behave differently. Using
Equations~\eqref{eq:app_salu_derivative_loga} and
\eqref{eq:app_salu_derivative_logb}, together with
$y\rightarrow S\operatorname{sign}(x)$, gives
\begin{equation}
\begin{aligned}
    \frac{\partial y}{\partial\widetilde a}
    &\longrightarrow
    \frac{S}{2}\operatorname{sign}(x),
    &
    \frac{\partial y}{\partial\widetilde b}
    &\longrightarrow
    -\frac{S}{2}\operatorname{sign}(x).
\end{aligned}
\label{eq:app_salu_log_gradient_saturation}
\end{equation}
Although both optimized-coordinate derivatives remain nonzero individually,
their leading terms are opposite and encode only the saturation-amplitude
direction. Indeed, Equation~\eqref{eq:app_salu_width_direction} and the
asymptotic width derivative give
\begin{equation}
    \frac{\partial\phi}{\partial\widetilde a}
    +
    \frac{\partial\phi}{\partial\widetilde b}
    \sim
    \frac{SW^2\operatorname{sign}(x)}{|x|^2}
    \longrightarrow0.
\label{eq:app_salu_width_signal_saturation}
\end{equation}
Thus, the independent signal that changes $W$ decays quadratically in the
optimized coordinates just as it does in the geometric coordinates. In the
strict saturation limit, the output approaches
$S\operatorname{sign}(x)$ and contains no local information about $W$; no
bijective reparameterization can recover information absent from the
function.

For the finite preactivations encountered in a network, the width signal is
generally nonzero and the $(\widetilde a,\widetilde b)$ coordinates express
the amplitude and width directions through coupled combinations of two
optimized variables. This changes the metric and scaling under which an
optimizer updates $(S,W)$, but does not by itself establish better
conditioning than direct optimization in geometric coordinates. Such a claim
would require a controlled empirical comparison of the two
parameterizations.

This is a pointwise differential property, not a guarantee of global
self-regulation or stable network optimization. The complete gradients of
$\widetilde a$ and $\widetilde b$ also depend on upstream loss derivatives,
the distribution of activations, and the surrounding architecture.

%% file: appendix_forward_variance.tex
\section{Proof of the Forward-Variance Theorem}
\label{app:forward_variance}

Recall the homogeneous SALU variance operator
\begin{equation}
    \mathcal{V}(v)
    =
    \sigma_w^2S^2
    \mathbb{E}_{Z}
    \left[
        \frac{vZ^2}{W^2+vZ^2}
    \right],
    \qquad
    Z\sim\mathcal{N}(0,1),\qquad v\geq0.
\label{eq:app_variance_operator}
\end{equation}

\subsection{Monotonicity, Concavity, and Boundedness}

\begin{lemma}
\label{lem:variance_operator_shape}
For $S,W,\sigma_w>0$, the mapping $\mathcal{V}$ is continuous, strictly
increasing, strictly concave, and uniformly bounded on $[0,\infty)$. Moreover,
\begin{equation}
    \mathcal{V}(0)=0,
    \qquad
    \mathcal{V}'(0)
    =
    \sigma_w^2\left(\frac{S}{W}\right)^2.
\end{equation}
\end{lemma}

\begin{proof}
For fixed $Z$, differentiation of the integrand gives
\begin{equation}
    \frac{\partial}{\partial v}
    \left(
        \frac{vZ^2}{W^2+vZ^2}
    \right)
    =
    \frac{W^2Z^2}{(W^2+vZ^2)^2}.
\end{equation}
Therefore,
\begin{equation}
    \mathcal{V}'(v)
    =
    \sigma_w^2S^2W^2
    \mathbb{E}_{Z}
    \left[
        \frac{Z^2}{(W^2+vZ^2)^2}
    \right]
    >0.
\label{eq:app_variance_first_derivative}
\end{equation}
Differentiating once more yields
\begin{equation}
    \mathcal{V}''(v)
    =
    -2\sigma_w^2S^2W^2
    \mathbb{E}_{Z}
    \left[
        \frac{Z^4}{(W^2+vZ^2)^3}
    \right]
    <0,
\label{eq:app_variance_second_derivative}
\end{equation}
which proves strict monotonicity and strict concavity. The differentiation
under the expectation is justified by domination on every compact interval
of $v\geq0$.

At the origin, using $\mathbb{E}[Z^2]=1$,
\begin{equation}
    \mathcal{V}'(0)
    =
    \frac{\sigma_w^2S^2}{W^2}
    =
    \sigma_w^2\left(\frac{S}{W}\right)^2.
\end{equation}
Finally,
\begin{equation}
    0
    \leq
    \frac{vZ^2}{W^2+vZ^2}
    <1
\end{equation}
for finite $v$, and hence
\begin{equation}
    0
    \leq
    \mathcal{V}(v)
    <
    \sigma_w^2S^2.
\end{equation}
This proves the claimed boundedness.
\end{proof}

\subsection{Proof of Theorem~\ref{thm:forward_variance}}

Define
\begin{equation}
    g(v)=\mathcal{V}(v)-v
\end{equation}
and
\begin{equation}
    \chi_0
    =
    \mathcal{V}'(0)
    =
    \sigma_w^2\left(\frac{S}{W}\right)^2.
\end{equation}
Since $\mathcal{V}$ is strictly concave, $g$ is also strictly concave, with
\begin{equation}
    g(0)=0,
    \qquad
    g'(0)=\chi_0-1.
\end{equation}

\paragraph{Contractive regime: $\chi_0<1$.}
In this case $g'(0)<0$. Because $g'$ is strictly decreasing,
\begin{equation}
    g'(v)<0,
    \qquad v>0.
\end{equation}
Consequently, $g(v)<0$ for every $v>0$, so the origin is the only fixed point.
For any $v_0>0$, the recursion satisfies
\begin{equation}
    0\leq v_{\ell+1}=\mathcal{V}(v_\ell)<v_\ell.
\end{equation}
The sequence is therefore decreasing and bounded below, and hence converges
to a limit $\bar v\geq0$. Continuity gives
\begin{equation}
    \bar v=\mathcal{V}(\bar v).
\end{equation}
Since zero is the unique fixed point, $\bar v=0$.

\paragraph{Critical regime: $\chi_0=1$.}
Here $g'(0)=0$. Strict concavity implies
\begin{equation}
    g'(v)<g'(0)=0,
    \qquad v>0,
\end{equation}
and therefore $g(v)<0$ for all $v>0$. The same monotone-convergence argument
shows that $v_\ell\to0$. Thus, the origin is marginal at linear order but
globally attracting through the nonlinear terms of the variance map.

This behavior can also be seen from the local expansion
\begin{equation}
    \mathcal{V}(v)
    =
    \chi_0 v
    -
    \frac{3\sigma_w^2S^2}{W^4}v^2
    +
    \mathcal{O}(v^3),
\label{eq:app_variance_local_expansion}
\end{equation}
where the negative quadratic term produces contraction when $\chi_0=1$.

\paragraph{Expansive regime: $\chi_0>1$.}
In this case $g'(0)>0$, so $g(v)>0$ for all sufficiently small $v>0$ and the
origin is unstable. On the other hand, boundedness gives
\begin{equation}
    g(v)
    <
    \sigma_w^2S^2-v
    \longrightarrow
    -\infty
    \qquad
    \text{as }v\to\infty.
\end{equation}
Continuity therefore guarantees at least one positive root $v^\star$.

This root is unique. Indeed, if two positive roots
$0<v_1<v_2$ existed, strict concavity together with $g(0)=g(v_2)=0$ would
imply
\begin{equation}
    g(v_1)>0,
\end{equation}
contradicting $g(v_1)=0$.

The positive equilibrium is locally stable. Since $\mathcal{V}$ is strictly
concave and $\mathcal{V}(0)=0$,
\begin{equation}
    0
    <
    \mathcal{V}'(v^\star)
    <
    \frac{\mathcal{V}(v^\star)-\mathcal{V}(0)}{v^\star-0}
    =1.
\end{equation}

It is also globally attracting on $(0,\infty)$. If $0<v<v^\star$, then
\begin{equation}
    v<\mathcal{V}(v)<v^\star,
\end{equation}
where the first inequality follows from $g(v)>0$ and the second from strict
monotonicity. Hence the iteration increases toward $v^\star$. If
$v>v^\star$, then
\begin{equation}
    v^\star<\mathcal{V}(v)<v,
\end{equation}
so the iteration decreases toward $v^\star$. This completes the proof.

Figure~\ref{fig:variance_fixed_points} illustrates the contractive and
expansive regimes using the same numerical variance operator as in the
theoretical analysis.

%% file: appendix_backward_propagation.tex
\section{Backward Gradient Propagation}
\label{app:backward_propagation}

\subsection{Backward Second-Moment Factor}

Under the standard mean-field independence approximation, backpropagation
through one homogeneous layer multiplies the gradient second moment by
\begin{equation}
\chi(v)=\sigma_w^2\E_Z\!\left[
\phi'(\sqrt v Z)^2\right].
\end{equation}
With $q_\ell$ defined as the expected squared gradient norm, this factor applies to the total second moment. A per-neuron second moment instead includes the width ratio $n_{\ell+1}/n_\ell$ for unequal-width layers under the stated weight scaling.

For SALU,
\begin{equation}
\phi'(x)=\frac SW
\left(1+\frac{x^2}{W^2}\right)^{-3/2},
\end{equation}
and therefore
\begin{equation}
\chi(v)=\sigma_w^2\left(\frac SW\right)^2M(v/W^2),
\qquad
M(u)=\E_Z[(1+uZ^2)^{-3}].
\label{eq:app_backward_factor}
\end{equation}

Differentiation under the expectation gives
\begin{equation}
M'(u)=-3\E_Z\!\left[\frac{Z^2}{(1+uZ^2)^4}\right]<0
\end{equation}
and
\begin{equation}
M''(u)=12\E_Z\!\left[\frac{Z^4}{(1+uZ^2)^5}\right]>0.
\end{equation}
Thus $M$ is strictly decreasing and convex, with $M(0)=1$. Dominated
convergence gives $M(u)\to0$ as $u\to\infty$. More precisely,
\begin{equation}
M(u)\sim\frac{3\pi}{8\sqrt{2\pi}}u^{-1/2},
\qquad u\to\infty,
\end{equation}
obtained by the substitution $t=\sqrt u\,Z$ and
$\int_{-\infty}^{\infty}(1+t^2)^{-3}\,dt=3\pi/8$.

\subsection{Limitation of the Homogeneous Regime}

\begin{proposition}[Limitation of the homogeneous regime]
\label{prop:salu_forward_backward_tradeoff}
Every positive forward fixed point of the homogeneous zero-bias SALU MLP
satisfies $\chi(v^\star)>1$.
\end{proposition}

\begin{proof}

Let $X\sim\mathcal N(0,v^\star)$, where $v^\star>0$ is a forward fixed
point. Since SALU is odd and the Gaussian distribution is symmetric,
\begin{equation}
\E[\phi(X)]=0.
\end{equation}
The fixed-point equation is therefore
\begin{equation}
v^\star
=\sigma_w^2\E[\phi(X)^2]
=\sigma_w^2\operatorname{Var}(\phi(X)).
\label{eq:app_forward_fixed_variance}
\end{equation}

The Gaussian Poincar\'e inequality states that
\begin{equation}
\operatorname{Var}(f(X))
\leq v^\star\E[f'(X)^2].
\end{equation}
Equality holds only for affine functions. SALU is nonlinear for finite
$S,W>0$, so the inequality is strict:
\begin{equation}
\operatorname{Var}(\phi(X))
<v^\star\E[\phi'(X)^2].
\end{equation}
Substituting Eq.~\eqref{eq:app_forward_fixed_variance} and dividing by
$v^\star>0$ gives
\begin{equation}
1
<\sigma_w^2\E[\phi'(X)^2]
=\chi(v^\star).
\end{equation}
\end{proof}

\subsection{Interpretation and Scope}

The derivative $\mathcal V'(v^\star)$ controls the local stability of the
scalar forward-variance iteration, whereas $\chi(v^\star)$ controls the
mean-field gradient second moment. Their definitions are different, and a
stable positive forward fixed point does not imply stable gradients. Indeed,
the proposition shows that this second moment has a multiplicative factor
greater than one at the positive fixed point of the homogeneous zero-bias
SALU MLP. This does not quantify gradient behavior in a trained SaluNet.

This conclusion depends on the analyzed setting: infinite width, Gaussian
preactivations, zero bias, homogeneous geometry, and the usual forward--backward
independence approximation. Residual paths, finite-width correlations,
convolutional weight sharing, attention, and learned layer-wise $(S_\ell,W_\ell)$
can alter the effective gradient dynamics. The proposition is therefore a
precise limitation of the mean-field model, not a stability claim for the
benchmark SaluNet architectures.

%% file: appendix_salu_gate_geometry.tex
\section{Geometry of the SALU-Based Gates}
\label{app:salu_gate_geometry}

Figure~\ref{fig:app_salu_based_gates} illustrates how the SALU coordinates
control SWALU and GALU. At fixed $W$, changing $S$ modifies both the gate
amplitude and its asymptotic branches. At fixed $S$, increasing $W$ broadens
the transition region and reduces the sensitivity of the internal SALU map
around the origin. The initialization $(S,W)=(1,1)$ produces gate profiles
close to the corresponding fixed Swish and GELU references while leaving their
geometry learnable during training.

\begin{figure}[H]
    \centering
    \includegraphics[width=\linewidth]
    {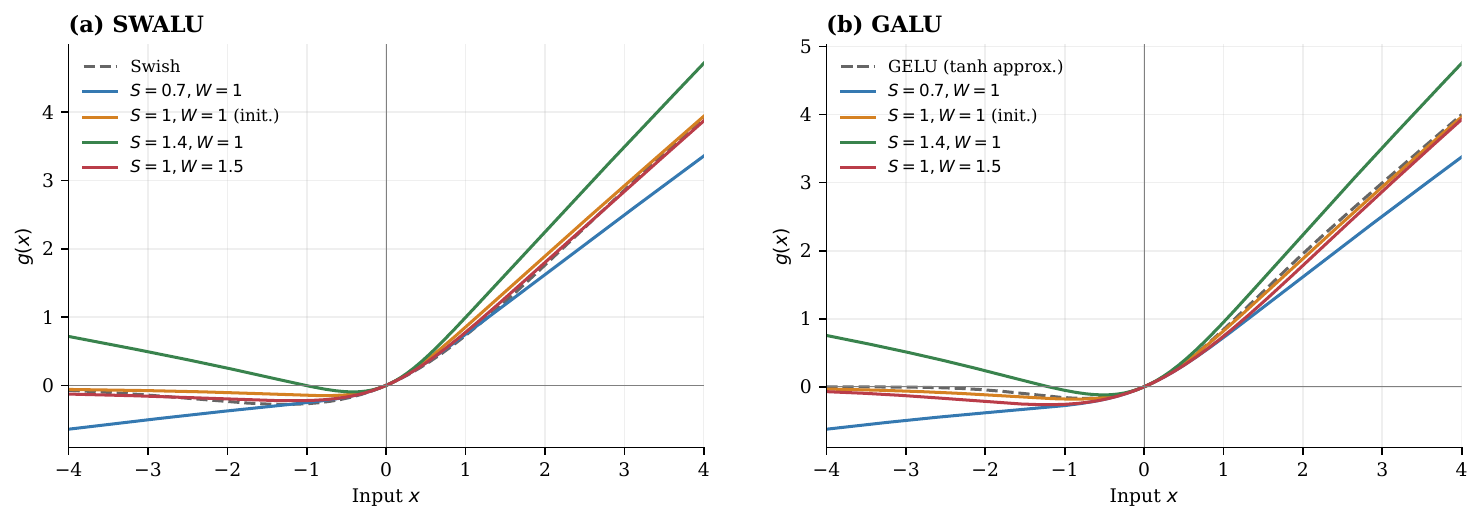}
    \caption{\textbf{Geometry of the SALU-based activation gates.}
    \textbf{(a)} SWALU for representative values of the saturation amplitude
    $S$ and transition width $W$, compared with Swish.
    \textbf{(b)} GALU for the same geometric configurations, compared with
    the tanh approximation of GELU. The configuration $(S,W)=(1,1)$ is the
    initialization used in our experiments. The remaining configurations are
    illustrative and are not parameters extracted from a trained model.}
    \label{fig:app_salu_based_gates}
\end{figure}

The asymptotic behavior also shows that $S$ controls more than the local
curvature of the gates. Let
\begin{equation}
    z(x)=x
\end{equation}
for SWALU and
\begin{equation}
    z(x)
    =
    \sqrt{\frac{2}{\pi}}
    \left(x+0.044715x^3\right)
\end{equation}
for GALU. Since $z(x)\rightarrow\pm\infty$ with the sign of $x$,
\begin{equation}
    \phi(z(x);S,W)
    \longrightarrow
    S\,\operatorname{sign}(x).
\end{equation}
Both gates therefore satisfy
\begin{equation}
\lim_{x\to+\infty}\frac{f(x;S,W)}{x}=\frac{1+S}{2},
\qquad
\lim_{x\to-\infty}\frac{f(x;S,W)}{x}=\frac{1-S}{2}.
\end{equation}
When $S=1$, the positive branch approaches the identity and the negative branch
is suppressed, reproducing the qualitative asymptotic behavior of Swish and
GELU. When $S\neq1$, the negative branch becomes asymptotically linear.
In particular, $S<1$ produces a positive asymptotic slope on the negative
branch, whereas $S>1$ produces a negative asymptotic slope. Consequently,
learning $S$ allows SWALU and GALU to move beyond the conventional
Swish/GELU gate family rather than merely adjusting their transition
sharpness.

%% file: appendix_architectural_details.tex
\section{Architectural and Implementation Details}
\label{app:architectural_details}

\subsection{General Replacement Rule}
\label{app:general_replacement_rule}

SaluNet modifies a reference architecture through two complementary
replacements:
\begin{equation}
    \operatorname{BN},\operatorname{LN}
    \longrightarrow
    \operatorname{SALU},
    \qquad
    \operatorname{ReLU},\operatorname{GELU},\operatorname{Swish}
    \longrightarrow
    \operatorname{SWALU}\ \text{or}\ \operatorname{GALU}.
\label{eq:app_salunet_replacement_rule}
\end{equation}
Unless stated otherwise, convolutional, linear, attention, residual,
downsampling, patch-embedding, pooling, and classification operators are
retained from the corresponding reference architecture.

Each SALU, SWALU, or GALU instance is controlled by one effective positive
pair $(a,b)$, equivalently represented by one geometric pair $(S,W)$. The
positivity mechanism is implementation-specific: SALU stabilizers use
log-parameters in the reported experiments; the CIFAR SWALU/GALU gates and
the SaluNet-T-B/16 SWALU gates use an absolute-value parameterization; and the
SaluNet-C-50 SWALU gates use log-parameters. The effective pair is shared across all
channels, spatial positions, or tokens processed by that instance. Hence,
each inserted operator introduces only two trainable scalars.

In the convolutional SaluNet-C family, sharing follows the residual-block
structure. The stem uses one activation gate, and each residual block uses one
SWALU or GALU instance repeatedly at all activation sites inside that block.
The SALU modules replacing normalization layers are not shared: every
normalization site receives its own SALU instance.

\subsection{Convolutional Backbones}
\label{app:convolutional_architectures}

\paragraph{SaluNet-C-18.}
SaluNet-C-18 is derived from the ResNet-18 implementation used in our CIFAR
training pipeline. Every Batch Normalization module in the convolutional stem,
residual branches, and projection shortcuts is replaced by SALU. Conventional
activations are replaced by SALU-based gates, with one gate in the stem and one
gate shared by the two activation sites of each BasicBlock. The resulting
network contains $20$ independently parameterized SALU stabilizers and nine
block-wise gates covering $17$ activation sites. We evaluate both SWALU and
GALU on both CIFAR-10 and CIFAR-100 and report both configurations. Residual connections, stage widths, convolutional kernels,
downsampling locations, global average pooling, and the classification head
are otherwise unchanged.

\paragraph{SaluNet-C-50.}
SaluNet-C-50 is derived from the ResNet-50 implementation used in our
ImageNet training pipeline. Every Batch Normalization module in the
convolutional stem, bottleneck branches, and projection shortcuts is replaced
by an independently parameterized SALU. Every ReLU activation site is
processed by SWALU, with one gate in the stem and one gate shared by the three
activation sites of each Bottleneck. The resulting network contains $53$ SALU
stabilizers and $17$ block-wise SWALU gates covering $49$ activation sites.
The bottleneck structure, residual connections, stage depths and widths,
convolutional kernels, downsampling locations, global average pooling, and
classification head are otherwise retained. The stem and projection-shortcut
SALU modules use the architecture-specific initializations reported in
Table~\ref{tab:salunet_c50_salu_initialization}; all remaining SALU modules
use the default initialization.

\subsection{Transformer Backbones}
\label{app:transformer_architectures}

\paragraph{SaluNet-T-B/16.}
SaluNet-T-B/16 is constructed from the \texttt{timm} implementation of
\texttt{vit\_base\_patch16\_224} used by the public Derf training
framework~\citep{chen2025derfrepository}.
The backbone employs pre-normalized Transformer blocks. We extend the
framework with a conversion rule that recursively replaces every instantiated
\texttt{LayerNorm} module with SALU. This includes the two normalization
operators within each Transformer block and any final normalization
instantiated by the backbone. The MLP activation is replaced by SWALU.
Attention modules, residual paths, embedding dimension, number of attention
heads, patch embedding, global-pooling strategy, stochastic depth, and
classification head are otherwise retained.

\paragraph{SaluNet-T-CIFAR.}
SaluNet-T-CIFAR uses our compact CIFAR implementation for both datasets.
Images are split into $4\times4$ patches, giving 64 patch tokens and one
class token. The backbone has seven pre-normalized Transformer blocks,
embedding dimension $384$, MLP hidden dimension $384$, and 12 attention
heads. Each head has dimension $d_h=32$, and attention logits are
$QK^\top/\sqrt{d_h}=QK^\top/\sqrt{32}$.
Each MLP contains exactly one activation between its two linear projections.
The two normalization sites per block and the final normalization are
replaced by independently parameterized SALU instances. The seven MLP activations use GELU, SWALU, or GALU, depending on
the configuration. All SALU-based configurations contain
$N_s=15$ independently parameterized stabilizers.
With SWALU or GALU, $N_g=7$, giving 44 trainable geometric scalars.
With GELU, $N_g=0$, giving 30 trainable geometric scalars. Baselines retain the same backbone and use the stabilizer and activation specified in Table~\ref{tab:vit_cifar_results}.

\subsection{Gate Comparisons and Interpretation}
\label{app:gate_selection}

SWALU and GALU are alternative SALU-based gates rather than operators assigned
to specific architecture families. Both can be used in convolutional and
Transformer backbones.

For SaluNet-C, preliminary comparisons showed no consistent benefit from
assigning an independent gate to every activation site; the block-wise shared
configuration matched or slightly exceeded its site-wise counterpart. We
therefore retain one shared SWALU or GALU per residual block throughout the
SaluNet-C family. The same sharing rule is used for both gates in the comparisons below.

Both SWALU and GALU are evaluated on SaluNet-C-18 and SaluNet-T-CIFAR
for both CIFAR datasets (Tables~\ref{tab:cifar_cnn_results}
and~\ref{tab:vit_cifar_results}). The observed ordering depends on the
architecture, dataset, and whether the main or EMA model is considered.

For SaluNet-C-18 on CIFAR-10, GALU has higher mean accuracies than SWALU:
$97.19\pm0.11\%$ versus $97.08\pm0.05\%$ with EMA, and
$97.17\pm0.10\%$ versus $97.05\pm0.06\%$ without EMA.
On CIFAR-100, SWALU has the higher EMA mean ($83.02\pm0.29\%$ versus
$82.95\pm0.31\%$), while the main-model means are nearly identical
($82.36\pm0.31\%$ for SWALU and $82.37\pm0.34\%$ for GALU).
These small differences do not establish a universal ordering or statistical
equivalence.

For SaluNet-T-CIFAR on CIFAR-10, SWALU has the highest mean in both
columns: $91.91\pm0.19\%$ with EMA and $91.92\pm0.16\%$ without EMA.
Relative to SALU+GELU, the EMA gains are $0.52$ points for SWALU and
$0.31$ points for GALU. On CIFAR-100, SALU+GELU has the highest EMA mean
($70.39\pm0.41\%$), followed by GALU ($70.29\pm0.11\%$) and SWALU
($70.11\pm0.48\%$). GALU has the highest main-model mean
($69.64\pm0.29\%$, versus $69.56\pm0.51\%$ for GELU and
$69.22\pm0.43\%$ for SWALU).
Thus, the adaptive gates provide context-dependent benefits. The existing
comparisons do not isolate the effect of learning their parameters; matched
learned-versus-fixed gate ablations would be required for that claim.

At ImageNet scale, we use SWALU for both SaluNet-C-50 and SaluNet-T-B/16.
This choice is made for computational efficiency rather than through an
exhaustive ImageNet-scale comparison of the two gates. GALU first evaluates
the cubic argument of the tanh-based GELU approximation,
\begin{equation}
    h(x)
    =
    \sqrt{\frac{2}{\pi}}
    \left(x+0.044715x^3\right),
\label{eq:app_galu_inner_map}
\end{equation}
and then applies SALU to $h(x)$, whereas SWALU applies SALU directly to $x$.
SWALU therefore has a lower arithmetic cost and was more practical under our
ImageNet-scale computational budget.

The evaluated gates are summarized in
Table~\ref{tab:gate_selection}.

\begin{table}[h]
\centering
\caption{SALU-based gates evaluated in each experimental configuration. The CNN
batch-size sweeps use GALU on CIFAR-10 and SWALU on CIFAR-100. ImageNet experiments
use SWALU for computational efficiency.}
\label{tab:gate_selection}
\begin{tabular}{@{}lll@{}}
\toprule
Architecture & Dataset & Gate \\
\midrule
SaluNet-C-18    & CIFAR-10    & SWALU, GALU  \\
SaluNet-C-18    & CIFAR-100   & SWALU, GALU \\
SaluNet-T-CIFAR & CIFAR-10    & SWALU, GALU  \\
SaluNet-T-CIFAR & CIFAR-100   & SWALU, GALU  \\
SaluNet-C-50    & ImageNet-1K & SWALU \\
SaluNet-T-B/16  & ImageNet-1K & SWALU \\
\bottomrule
\end{tabular}
\end{table}

\subsection{Parameter Initialization}
\label{app:salu_initialization}

\paragraph{SALU-based gates.}
Every SWALU and GALU gate is initialized with
\begin{equation}
    a_0=b_0=1,
\label{eq:app_gate_initialization_ab}
\end{equation}
corresponding to the geometric initialization
\begin{equation}
    S_0=W_0=1.
\label{eq:app_gate_initialization_sw}
\end{equation}
The internal SALU primitive therefore starts from the smooth, bounded, odd
profile
\begin{equation}
    \phi_0(x)
    =
    \frac{x}{\sqrt{1+x^2}}.
\label{eq:app_initial_gate_primitive}
\end{equation}
The corresponding modulation factor
\begin{equation}
    \frac{1+\phi_0(x)}{2}
\end{equation}
lies in $(0,1)$ and equals $1/2$ at the origin. Consequently, SWALU starts
from a smooth Swish-like modulation profile, whereas GALU starts from a smooth
GELU-like modulation profile. This provides a familiar initial geometry while
allowing each gate instance to adapt its amplitude and transition width during
training.

The mathematical SWALU and GALU functions are unbounded because the
modulation factor is multiplied by $x$. The reference SaluNet-C-50
implementation additionally clamps SWALU outputs to $[-10,10]$,
as detailed in Appendix~\ref{app:imagenet_protocols}. Moreover, after initialization, the learned
amplitude $S$ is not constrained to remain below one. The modulation factor is
therefore not required to remain in $(0,1)$ throughout training; SWALU and
GALU should be interpreted as adaptive multiplicative modulators rather than
probability-valued gates.

\paragraph{Default SALU stabilizer initialization.}
Unless stated otherwise, SALU stabilizers are initialized with
\begin{equation}
    a_0=1,
    \qquad
    b_0=0.1.
\label{eq:salu_default_initialization}
\end{equation}
In geometric coordinates, this corresponds to
\begin{equation}
    S_0=W_0=\sqrt{10},
    \qquad
    \frac{S_0}{W_0}=1.
\label{eq:app_salu_default_geometry}
\end{equation}
Thus, the default initialization has unit central gain while placing the
characteristic transition at $W_0=\sqrt{10}$.

\paragraph{SaluNet-C-50 exceptions.}
For SaluNet-C-50, preliminary experiments indicated that the stem and
projection shortcuts benefited from architecture-specific initial geometries.
We therefore use the SALU initializations reported in
Table~\ref{tab:salunet_c50_salu_initialization} at these locations. All other
SALU modules in SaluNet-C-50 use the default initialization
$(a,b)=(1,0.1)$.

\begin{table}[h]
\centering
\caption{Architecture-specific initialization of SALU stabilizers in
SaluNet-C-50. Implementation indices identify the corresponding modules in
the released conversion code.}
\label{tab:salunet_c50_salu_initialization}
\begin{tabular}{@{}lccccc@{}}
\toprule
Location & Index & $a_0$ & $b_0$ & $S_0$ & $W_0$ \\
\midrule
Stem SALU
    & 1  & 5.5 & 0.00001 & 741.62 & 134.84 \\
Stage 1 projection
    & 5  & 2.0 & 0.1    & 4.47  & 2.24 \\
Stage 2 projection
    & 15 & 1.5 & 0.01   & 12.25 & 8.16 \\
Stage 3 projection
    & 28 & 0.8 & 0.001  & 28.28 & 35.36 \\
Stage 4 projection
    & 47 & 1.0 & 0.001  & 31.62 & 31.62 \\
\bottomrule
\end{tabular}
\end{table}

The geometric values in
Table~\ref{tab:salunet_c50_salu_initialization} follow from
\begin{equation}
    S_0=\sqrt{\frac{a_0}{b_0}},
    \qquad
    W_0=\frac{1}{\sqrt{a_0b_0}}.
\label{eq:app_initial_geometry_conversion}
\end{equation}
These values specify only the starting geometry; all corresponding SALU
parameters follow their configured optimization schedule; a terminal zero learning rate, when enabled, freezes them while ordinary weights continue to learn.

For the benchmark architectures, except for the SaluNet-C-50 locations listed in
Table~\ref{tab:salunet_c50_salu_initialization}, all SALU stabilizers are
initialized with $(a,b)=(1,0.1)$. Every SWALU or GALU gate is initialized
with $(a,b)=(1,1)$. Network weights retain the default initialization of their
respective reference implementations.

The separate small-CNN diagnostics in Appendix~\ref{app:activation_plasticity} use their own initialization, $(a,b)=(1,1)$, and constant learning rate.

%% file: appendix_salu_learning_analysis.tex
\section{Learned SALU and SWALU Geometry}
\label{app:salu_learning_analysis}

This section reports a repeated-run ablation of the SALU coordinates,
followed by a separate geometry analysis of the previously logged
SaluNet-C-18/CIFAR-100 run attaining $83.27\%$ EMA Top-1 accuracy.
The updated ablation aggregates and the historical trajectory analysis are
distinct experiment records.

\subsection{Ablation of the SALU Coordinates}
\label{app:salu_parameter_ablation}

SALU contains two effective positive coordinates,
\begin{equation}
    \phi(x;a,b)=\frac{ax}{\sqrt{1+abx^2}},
\end{equation}
where $a$ determines the central gain and, together with $b$, determines the
saturation amplitude and transition width. To test whether both degrees of
freedom are useful, we compare the complete model with three restricted SALU
variants. The SWALU gates use the same parameterization and learning-rate schedule in every configuration, and
each configuration uses three independent runs under the same training recipe. Only the parameters of the
SALU modules replacing normalization are changed.

\begin{table}[h]
\centering
\caption{Controlled ablation of the SALU coordinates in
SaluNet-C-18/CIFAR-100. The SWALU parameterization and schedule are unchanged across configurations. Values are mean $\pm$ sample standard deviation of per-run
best Top-1 accuracies over three independent runs. The fully learned row
reproduces the SALU+SWALU results in Table~\ref{tab:cifar_cnn_results}.}
\label{tab:salu_coordinate_ablation}
\begin{tabular}{@{}llcc@{}}
\toprule
SALU $a$ & SALU $b$ & Main (\%) & EMA (\%) \\
\midrule
Learned & Learned & $\mathbf{82.36\pm0.31}$ & $\mathbf{83.02\pm0.29}$ \\
Fixed at $1$ & Fixed at $1$ & $81.40\pm0.22$ & $82.16\pm0.15$ \\
Learned & Fixed at $1$ & $81.54\pm0.31$ & $82.11\pm0.23$ \\
Fixed at $1$ & Learned & $81.24\pm0.28$ & $82.01\pm0.25$ \\
\bottomrule
\end{tabular}
\end{table}

Table~\ref{tab:salu_coordinate_ablation} shows that jointly learning $a$
and $b$ gives the highest mean in both model variants. Relative to fixed
SALU, the gains are $0.96$ points for the main model and $0.86$ points with
EMA. Relative to learning only $a$ or only $b$, the EMA gains are $0.91$
and $1.01$ points, respectively. The restricted variants have close observed
means; neither reproduces the mean accuracy of the two-coordinate model.
This repeated-run comparison favors the complete two-coordinate configuration under
this recipe, without establishing a universal ordering of the restricted
variants. Since the SWALU parameterization and optimization schedule are unchanged, it does not isolate the
benefit of learning the gate parameters themselves.

\paragraph{Initialization caveat.} The complete model uses the default SALU initialization $b_0=0.1$, whereas fixing $b=1$ also changes the initial geometry. These results compare parameter constraints and initial geometries; they do not isolate the effect of learning $b$ at matched initialization. A matched-initialization ablation is needed for that attribution.

\input{appendix_learned_salu_geometry}
\input{appendix_learned_swalu_geometry}
\input{appendix_salu_operating_regions}

%% file: appendix_learned_salu_geometry.tex
\subsection{Learned Layer-Wise SALU Geometry}
\label{app:learned_salu_geometry}

We examine how the SALU stabilizers specialize during the training of
SaluNet-C-18 on CIFAR-100. This analysis uses the analyzed run attaining
$83.27\%$ EMA Top-1 accuracy. For every SALU module replacing Batch
Normalization, the implementation logs the effective positive coordinates

\begin{equation}
    a=\exp(\widetilde a),
    \qquad
    b=\exp(\widetilde b)
\end{equation}

after each epoch. We convert the 20 learned pairs to the geometric coordinates

\begin{equation}
    S=\sqrt{\frac{a}{b}},
    \qquad
    W=\frac{1}{\sqrt{ab}},
    \qquad
    \frac{S}{W}=a.
\label{eq:app_logged_salu_geometry}
\end{equation}

The SWALU gate parameters are not included in this figure, so that the analysis
isolates the operators used for point-wise stabilization.
The final functions and summary statistics below use the end of the 300-epoch training run. This endpoint is distinct from the earlier epoch at which the best EMA
checkpoint attains the reported accuracy.

\begin{figure}[h]
    \centering

    \includegraphics[width=0.85\linewidth]
    {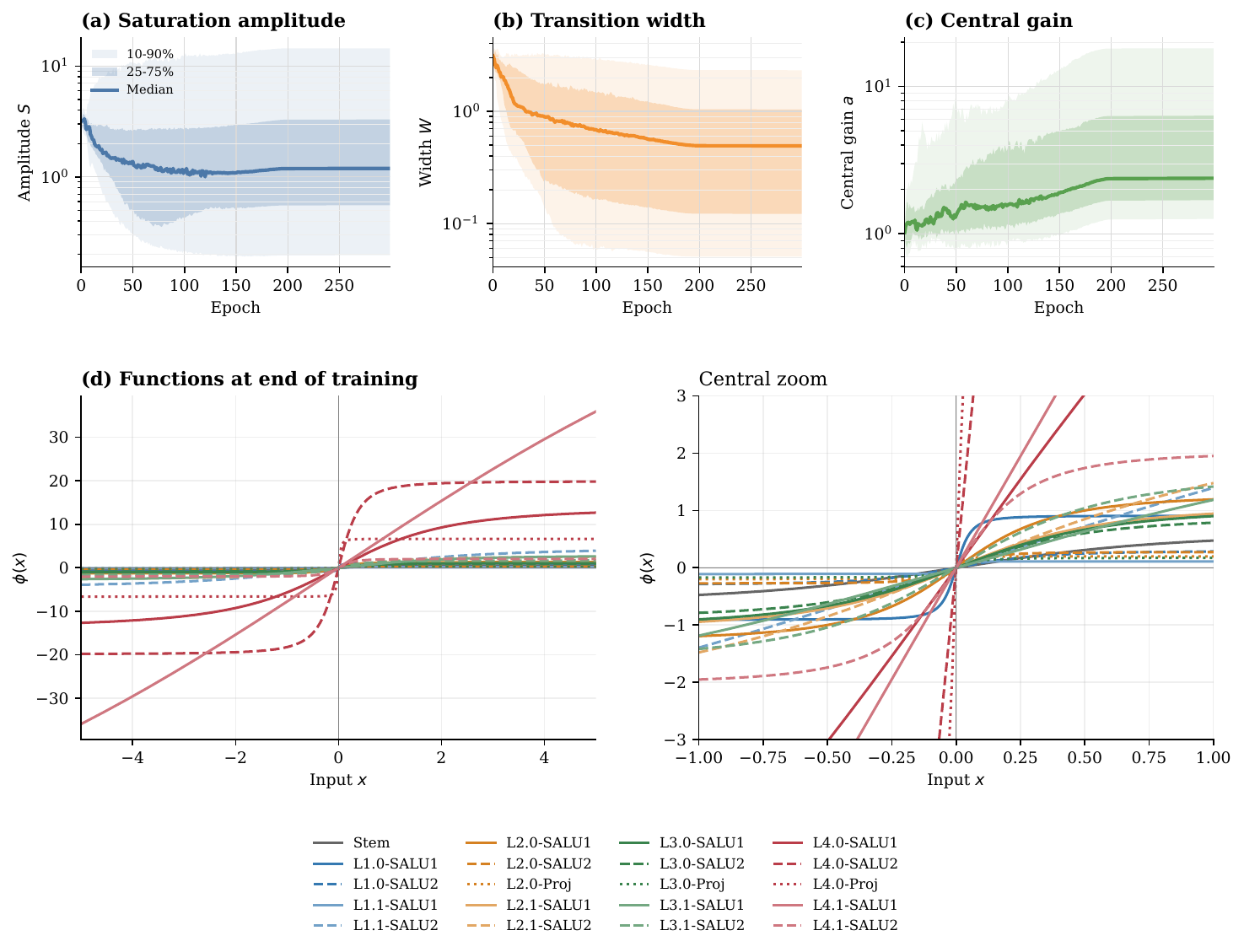}
    \caption{\textbf{Learned SALU geometry in SaluNet-C-18/CIFAR-100.}
    \textbf{(a--c)} Median saturation amplitude $S$, transition width $W$,
    and central gain $S/W=a$ across the 20 SALU stabilizers during training.
    Dark bands show the interquartile range and light bands show the
    10th--90th percentile range across layers. Logarithmic vertical axes
    preserve the layer-wise dispersion. \textbf{(d)} The 20 learned SALU
    functions at the end of training, with a central zoom. Colors identify the stem and
    residual stages; lighter shades denote the second block. Solid, dashed,
    and dotted lines denote SALU1, SALU2, and projection-path modules,
    respectively. The run reaches a best EMA Top-1 accuracy
    of $83.27\%$ at an earlier training epoch. Curves use the logged parameters at the end of training,
    rather than parameters extracted from that best checkpoint.}
    \label{fig:learned_salu_geometry}
\end{figure}

All layers begin near the common initialization
$S=W=\sqrt{10}$ and central gain $S/W=1$. Their trajectories subsequently
separate. At the end of training, the median amplitude, width, and central
gain are $1.193$, $0.492$, and $2.382$, respectively, while the corresponding
interquartile ranges are $S\in[0.555,3.297]$ and
$W\in[0.122,1.038]$. The final layer-wise comparison further shows that this
dispersion is structured across depth: several late-stage modules select
substantially larger amplitudes, whereas other modules operate with narrow,
low-amplitude geometries.
The function plots make this diversity explicit: some modules saturate
within the displayed input interval, whereas others retain a broad
approximately linear region. The central zoom exposes differences in local
gain that are less visible on the full output scale. Numerical values and
the corresponding geometric bar charts are provided in
Appendix~\ref{app:final_geometry_values}.

These trajectories demonstrate that the network does not retain a single
global SALU shape. Instead, modules initialized from nearly the same operator
learn distinct vertical scales, horizontal scales, and local gains according
to their architectural positions. This observation is consistent with the
intended role of SALU as a layer-wise adaptive geometry rather than a fixed
bounded activation. It does not by itself establish that any individual
trajectory causes the final accuracy; the controlled parameter ablation in
Appendix~\ref{app:salu_parameter_ablation} provides complementary evidence
for the utility of learning both geometric degrees of freedom.

%% file: appendix_learned_swalu_geometry.tex
\subsection{Learned Geometry of the SWALU Gates}
\label{app:learned_swalu_geometry}

We separately analyze the SALU primitive embedded in each SWALU gate of the
same SaluNet-C-18/CIFAR-100 run. The distinction from the preceding analysis
is important: the SALU modules in
Appendix~\ref{app:learned_salu_geometry} replace normalization layers and act
as bounded point-wise stabilizers, whereas the primitive considered here
modulates the unbounded gate

\begin{equation}
    \operatorname{SWALU}(x)
    =
    \frac{x}{2}\left[1+\phi_g(x;a_g,b_g)\right].
\end{equation}

The nine block-wise SWALU modules optimize direct parameters. One belongs to
the stem, and each of the eight residual BasicBlocks reuses one module at its
two activation sites. For this logged-run analysis, the plotting code maps
the raw parameters to the following effective positive coordinates:

\begin{equation}
    a_g=|a_{g,\mathrm{raw}}|+\epsilon_p,
    \qquad
    b_g=|b_{g,\mathrm{raw}}|+\epsilon_p,
    \qquad \epsilon_p=10^{-6}.
\end{equation}

The parameter offset is included once when computing the
geometry of the internal SALU primitive as

\begin{equation}
    S_g=\sqrt{\frac{a_g}{b_g}},
    \qquad
    W_g=\frac{1}{\sqrt{a_gb_g}},
    \qquad
    \frac{S_g}{W_g}=a_g.
\label{eq:app_logged_swalu_geometry}
\end{equation}

Here, $S_g$ bounds the internal primitive $\phi_g$, not the complete SWALU
output, which remains unbounded because it is multiplied by $x$.

\begin{figure}[h]
    \centering

    \includegraphics[width=0.85\linewidth]
    {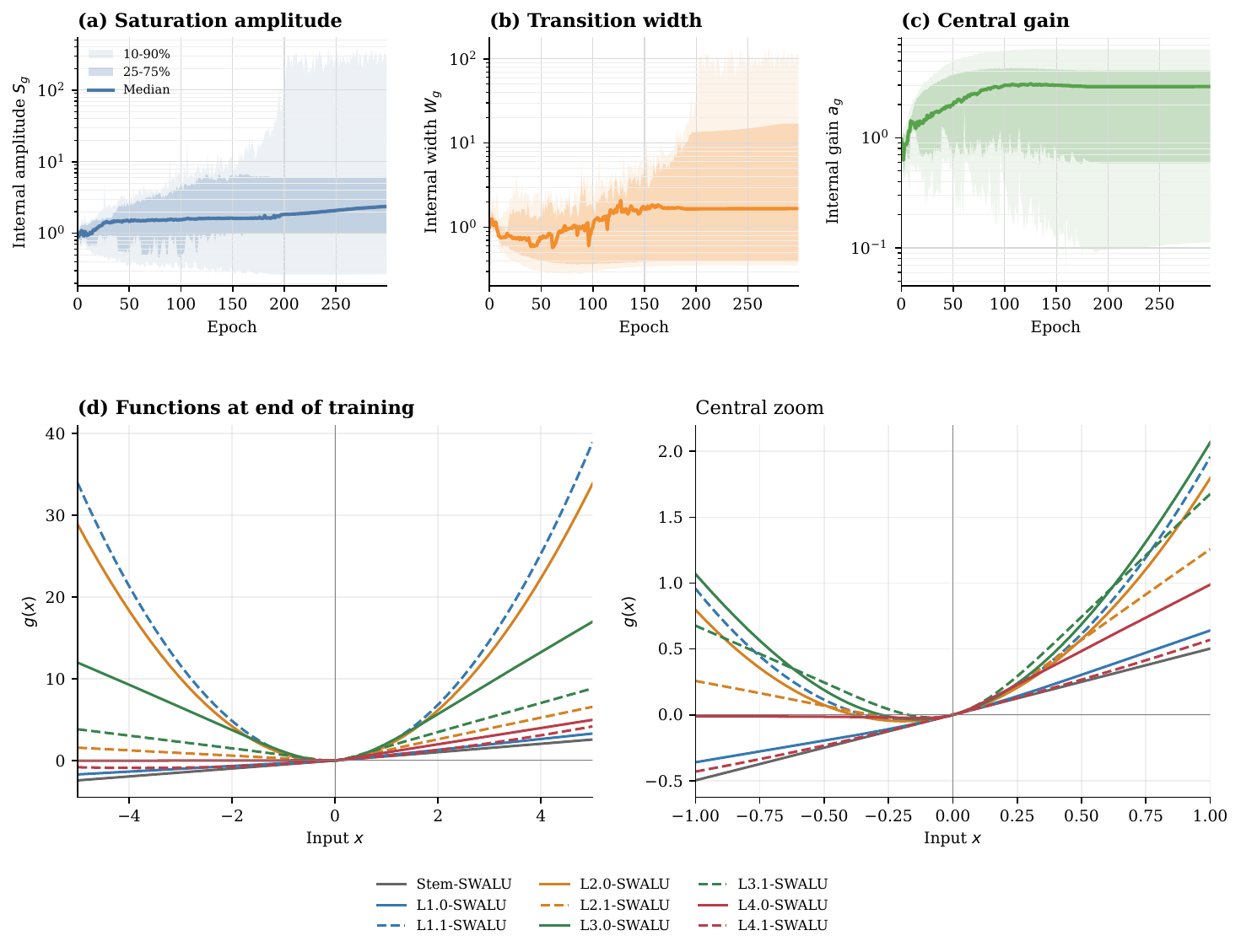}
    
    \caption{\textbf{Learned internal geometry of the SWALU gates in
    SaluNet-C-18/CIFAR-100.}
    \textbf{(a--c)} Median internal amplitude $S_g$, width $W_g$, and central
    gain $S_g/W_g=a_g$ across the nine block-wise shared SWALU modules. Dark
    bands show the interquartile range and light bands the 10th--90th
    percentile range across gates. The stem has its own gate, and each residual
    block shares one gate across its two activation sites. \textbf{(d)} The
    nine complete SWALU functions at the end of training, with a central zoom. Colors
    identify the network stage; dashed lines denote the second block.
    These plots show the full gate output, rather than the internal SALU
    primitive alone. The analyzed run reaches a best EMA Top-1 accuracy of
    $83.27\%$ at an earlier training epoch; the displayed functions use the final logged
    parameters at the end of training.}
    \label{fig:learned_swalu_geometry}
\end{figure}

Although all gates are initialized with $a_g=b_g=1$, their internal
geometries become strongly layer dependent. At the end of training, the
median values are

\begin{equation}
    \operatorname{median}(S_g)=2.366,
    \qquad
    \operatorname{median}(W_g)=1.669,
    \qquad
    \operatorname{median}(a_g)=2.916.
\end{equation}

The interquartile ranges are $S_g\in[0.989,6.105]$ and
$W_g\in[0.401,16.985]$, but several gates select substantially more extreme
operating regimes. In particular, a nearly vanishing $b_g$ can make both
$S_g$ and $W_g$ large while leaving their ratio $S_g/W_g=a_g$ finite. For
inputs satisfying $a_gb_gx^2\ll1$, the internal primitive then behaves as

\begin{equation}
    \phi_g(x)\approx a_gx,
\end{equation}

so the large nominal amplitude does not indicate that the gate is strongly
saturated. Instead, it indicates a broad approximately linear internal
region. Other gates learn smaller widths and therefore enter their nonlinear
transition at lower input magnitudes.
In this broad internal linear regime, the full gate is approximately
$\operatorname{SWALU}(x)\approx x/2+a_gx^2/2$, explaining the curved,
approximately quadratic profiles visible over the plotted interval.
Numerical effective parameters and geometric bar charts appear in
Appendix~\ref{app:final_geometry_values}.

The resulting diversity suggests that SWALU does not learn one global gating
shape across the network. Its block-level gates select different extents and
gains for their internal modulation, ranging from broad quasi-linear
primitives to more rapidly transitioning gates. Parameters remain shared
between the two activation sites within a BasicBlock. As with the SALU
trajectories, this is a
descriptive single-run analysis and should not be interpreted as establishing
that an individual gate trajectory causes the final accuracy.
\clearpage

%% file: appendix_salu_operating_regions.tex
\subsection{Layer-Wise Operating Regions}
\label{app:salu_operating_regions}

We next relate the learned SALU width to the preactivation distribution seen
by each stabilizing module. Measurements use the complete CIFAR-100 test set
and the EMA checkpoint of the analyzed SaluNet-C-18 run. The reconstructed
model attains $83.24\%$ Top-1 accuracy on reevaluation, compared with the historical logged best value of $83.27\%$. The cause of this $0.03$-point difference has not been established; the measurements below describe the reevaluated checkpoint. SWALU gates are excluded
from this analysis so that the measurements describe only the SALU modules
replacing normalization layers.

For a preactivation $x$, define the central-region coverages
\begin{equation}
    C_W=\Pr(|x|\leq W),
    \qquad
    C_{2W}=\Pr(|x|\leq 2W),
\end{equation}
and the normalized derivative retention
\begin{equation}
    D
    =
    \mathbb E\!\left[
        \left(1+\frac{x^2}{W^2}\right)^{-3/2}
    \right].
\label{eq:normalized_derivative_retention}
\end{equation}
Since
\begin{equation}
    \phi'(x)
    =
    \frac{S}{W}
    \left(1+\frac{x^2}{W^2}\right)^{-3/2},
\end{equation}
$D$ measures the mean retained sensitivity relative to the central gain
$S/W$. It isolates the position of the inputs within the learned geometry;
the actual mean local derivative is $(S/W)D$.

Figure~\ref{fig:salu_operating_regions} shows that SALU does not learn a
single operating regime shared by all layers. Fifteen of the twenty analyzed
stabilizers place at least $95\%$ of their preactivations within $|x|\leq W$,
and thirteen place at least $99\%$ there. The median values are
$C_W=99.98\%$ and $D=0.94$, indicating broad central coverage and substantial
retained sensitivity in most layers. However, two modules have
$C_W<25\%$. These modules operate much farther into the nonlinear region,
rather than simply expanding $W$ to contain their input distributions.

This behavior is compatible with SALU's gradual approach to saturation.
The interval $[-W,W]$ identifies a characteristic central region, not a hard
saturation boundary. For $|x|\gg W$,
\begin{equation}
    \phi'(x)
    \sim
    \frac{SW^2}{|x|^3},
\end{equation}
so sensitivity decreases polynomially. A layer can therefore place most
inputs in the nearly linear, high-sensitivity region when that behavior is
useful, or expose a larger fraction to stronger nonlinear shaping without the
rapid derivative collapse associated with exponentially or Gaussianly
saturating maps.

The dashed quantities computed with the initial width $W_0=\sqrt{10}$ are
local counterfactuals evaluated on the \emph{final learned representations};
they are not results from a separately trained model and should not be
interpreted as a causal baseline. More generally, this final-checkpoint
measurement establishes layer-wise specialization but does not by itself show
the temporal mechanism by which that specialization emerges. Together with
the controlled coordinate ablation in
Table~\ref{tab:salu_coordinate_ablation}, it supports the narrower conclusion
that jointly learned SALU geometry provides useful layer-specific operating
regimes.

\begin{figure}[h]
    \centering
    \includegraphics[width=\linewidth]
    {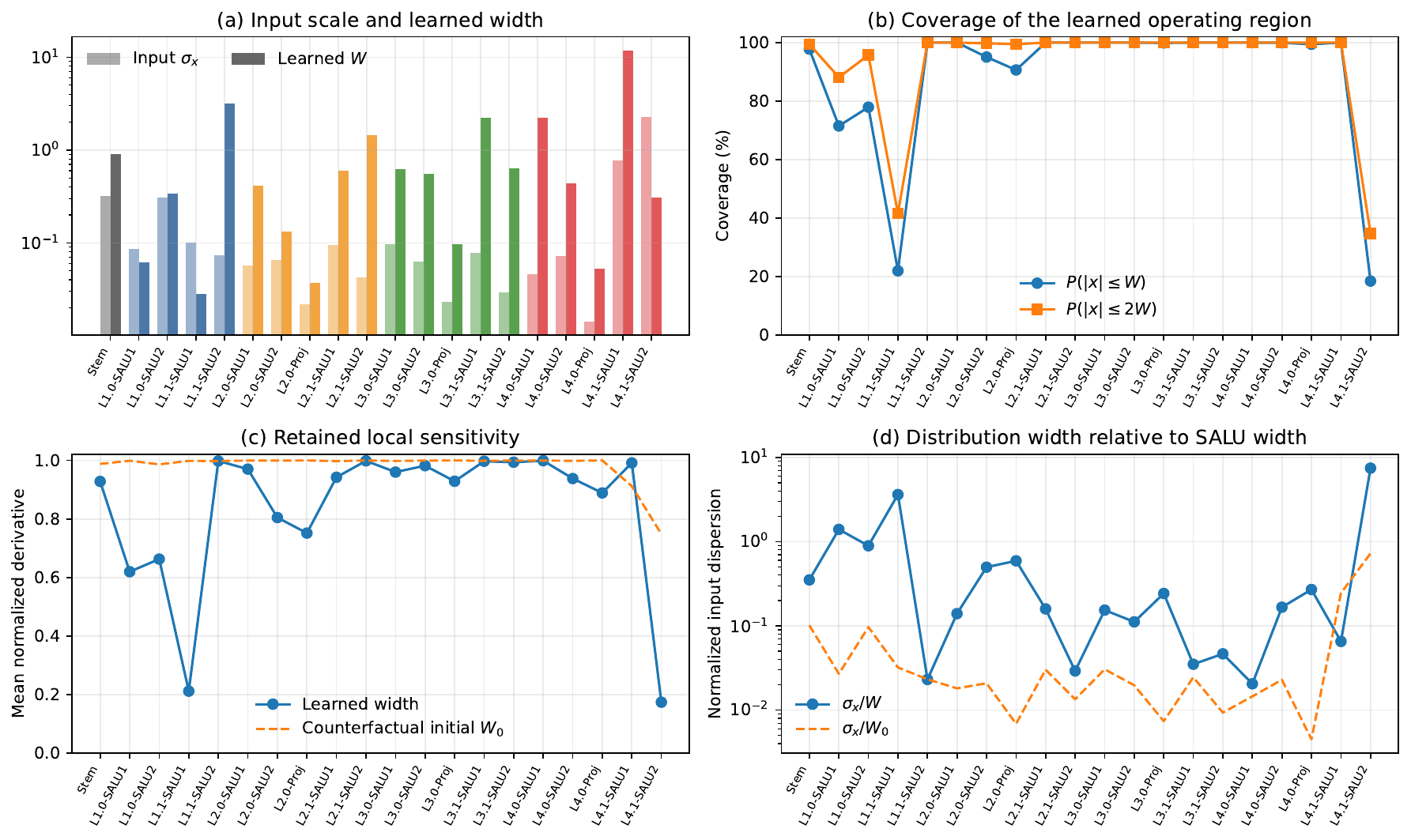}
    \caption{\textbf{Layer-wise operating regions of the SALU stabilizers in
    SaluNet-C-18/CIFAR-100.}
    \textbf{(a)} Standard deviation of each module's preactivations and its
    learned transition width $W$.
    \textbf{(b)} Fractions of preactivations satisfying $|x|\leq W$ and
    $|x|\leq2W$.
    \textbf{(c)} Mean normalized derivative retention $D$ from
    Equation~\eqref{eq:normalized_derivative_retention}. The dashed curve is
    the counterfactual value obtained by applying the initial width $W_0$ to
    the final learned representations.
    \textbf{(d)} Input dispersion relative to the learned width and to $W_0$.
    Most layers retain broad central coverage, whereas selected layers operate
    farther into SALU's nonlinear region.}
    \label{fig:salu_operating_regions}
\end{figure}

%% file: appendix_representation_analysis.tex
\section{Representation Geometry and Activation Statistics}
\label{app:representation_analysis}

\subsection{Checkpoints and Feature Extraction}
\label{app:representation_protocol}

We compare one trained SaluNet-C-18 checkpoint with one BN+ReLU ResNet-18
checkpoint on the complete CIFAR-100 test set. Layer1--Layer4 denote the
outputs of the four residual stages, recorded after the last residual block
of each stage. For a convolutional activation
$A\in\mathbb{R}^{N\times C\times H\times W}$, spatial positions are treated
as observations and channels as coordinates, yielding a matrix
$X\in\mathbb{R}^{NHW\times C}$. The complete activation stream is used for
the scalar moments. For the covariance spectrum, we select $20{,}000$
spatial observations using three deterministic spatial samplings, with
identical indices for the two models. After global average pooling, all
$10{,}000$ test representations are used, without spatial sampling.

Let $\lambda_1,\ldots,\lambda_d$ be the eigenvalues of the centered feature
covariance and define $p_i=\lambda_i/\sum_j\lambda_j$. We compute
\begin{equation}
    \operatorname{NER}(X)
    =\frac{1}{d}\exp\!\left(-\sum_{i=1}^{d}p_i\log p_i\right),
    \qquad
    \mathcal I(X)
    =\frac{(\sum_i\lambda_i)^2}{d\sum_i\lambda_i^2}.
\label{eq:app_representation_metrics}
\end{equation}
NER is the entropy-based effective rank~\citep{roy2007effective} divided by
the feature dimension. The index $\mathcal I$ is a normalized participation
ratio. Both lie in $(0,1]$ and increase as spectral energy is distributed
across more feature directions. Neither metric is, by itself, a measure of
classification quality.

\subsection{Spectral Geometry}
\label{app:representation_spectrum}

Table~\ref{tab:app_representation_geometry} gives the mean spectral estimates
over the three deterministic spatial-sampling seeds. The models remain close
at the first three residual stages. BN+ReLU then undergoes a pronounced
late-stage spectral contraction, whereas SaluNet retains a broader spectrum.
The separation remains visible after global average pooling.

\begin{table}[h]
\centering
\caption{Normalized effective rank and isotropy on CIFAR-100. Convolutional
entries are means over three deterministic spatial-sampling seeds. Pooled
entries use all test examples.}
\label{tab:app_representation_geometry}
\small
\begin{tabular}{@{}lcccc@{}}
\toprule
& \multicolumn{2}{c}{BN+ReLU} & \multicolumn{2}{c}{SaluNet-C-18} \\
Layer & NER & $\mathcal I$ & NER & $\mathcal I$ \\
\midrule
Conv1  & 0.0786 & 0.0703 & 0.0827 & 0.0676 \\
Layer1 & 0.6733 & 0.5066 & 0.6969 & 0.5512 \\
Layer2 & 0.6711 & 0.4748 & 0.6455 & 0.4492 \\
Layer3 & 0.7454 & 0.5328 & 0.8022 & 0.5781 \\
Layer4 & 0.1707 & 0.1258 & 0.5183 & 0.3203 \\
Pooled & 0.1615 & 0.1209 & 0.2985 & 0.2010 \\
\bottomrule
\end{tabular}
\end{table}

The sampling variability at Layer~4 is small: NER is
$0.1707\pm0.0003$ for BN+ReLU and $0.5183\pm0.0006$ for SaluNet, while
$\mathcal I$ is $0.1258\pm0.0004$ and $0.3203\pm0.0007$, respectively.
These standard deviations measure only sensitivity to the sampled spatial
positions. They are not estimates of variability across independently trained
networks.

\paragraph{Interpretation.}
The two spectral metrics describe complementary aspects of feature usage.
NER measures how many covariance directions are effectively occupied, whereas
$\mathcal I$ measures how evenly the representation energy is distributed
among those directions. Their joint increase at Layer~4 is therefore more
informative than either increase alone. Relative to BN+ReLU, SaluNet uses
approximately $3.0\times$ the normalized effective rank and
$2.5\times$ the isotropy at this stage. This indicates that its late
representation is not merely higher-energy; its covariance is supported by a
larger and less concentrated set of channel directions.

The location of the separation is also informative. Both networks maintain a
broad spectrum at Layer~3, but BN+ReLU contracts sharply at Layer~4, whereas
SaluNet largely preserves the occupied subspace. A concentrated late-stage
spectrum can make many channels redundant and force the classifier to rely on
a smaller number of dominant feature combinations. By retaining more balanced
directions, SaluNet can expose a richer collection of complementary cues to
the classifier and reduce domination by a few covariance modes. The persistence
of this effect after global pooling suggests that it reaches the actual input
of the linear classification head rather than remaining a purely spatial
phenomenon.

Higher rank and isotropy are not universally beneficial: additional directions
may encode nuisance variation rather than class-discriminative information.
Accordingly, the present result offers a plausible explanation for SaluNet's
accuracy gain only when read jointly with the observed classification result
and pooled-moment behavior. Establishing the mechanism more strongly requires
reproducing the spectral separation over independently trained models and
testing whether its magnitude covaries with accuracy.

\subsection{Activation Moments}
\label{app:activation_moments}

For scalar activation samples with mean $\mu$ and standard deviation $\sigma$,
we report
\begin{equation}
 \operatorname{Var}(x)=\mathbb E[(x-\mu)^2],\qquad
 \gamma=\frac{\mathbb E[(x-\mu)^3]}{\sigma^3},\qquad
 \kappa=\frac{\mathbb E[(x-\mu)^4]}{\sigma^4}-3,
\label{eq:app_activation_moments}
\end{equation}
where $\gamma$ is skewness and $\kappa$ is excess kurtosis. These quantities
are computed over all test images, channels, and spatial positions at each
convolutional stage.

\begin{figure}[h]
    \centering
    \includegraphics[width=\linewidth]
    {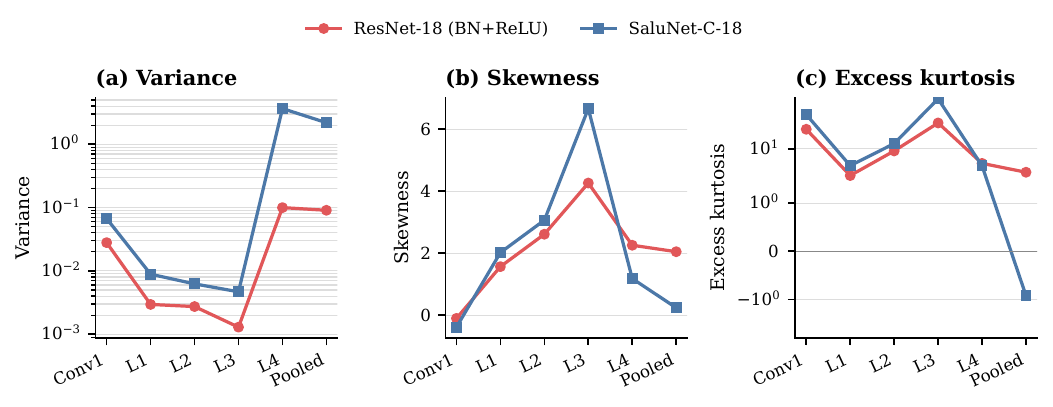}
    \caption{\textbf{Global activation moments across depth.}
    Variance is displayed logarithmically; excess kurtosis uses a symmetric
    logarithmic scale. SaluNet does not reproduce BN-like intermediate
    statistics. After global pooling, however, its representation is more
    symmetric and less heavy-tailed in the analyzed checkpoint.}
    \label{fig:app_activation_moments}
\end{figure}

SaluNet has larger variance at every measured stage, with the largest
difference at Layer~4 ($3.656$ versus $0.0999$). Its intermediate features
can also be more asymmetric and heavy-tailed. At Layer~3, skewness changes
from $4.26$ under BN+ReLU to $6.66$ under SaluNet, and excess kurtosis from
$30.13$ to $83.96$. Stable training in this checkpoint therefore does not
coincide with unit-variance or Gaussian intermediate activations.

After global average pooling, the trend changes: skewness decreases from
$2.05$ with BN+ReLU to $0.24$ with SaluNet, while excess kurtosis decreases
from $3.71$ to $-0.92$. The pooled SaluNet representation is consequently
more symmetric and less heavy-tailed, although it is not claimed to be
Gaussian.

\paragraph{Interpretation.}
Taken together with the spectral results, these moments suggest a plausible
representation-level explanation for the accuracy difference. SaluNet does
not force every intermediate stage toward unit variance or Gaussian-like
statistics. Instead, its learned pointwise geometry permits large and
heavy-tailed internal responses while controlling how each SALU module reacts
to them. This regime is accompanied by the preservation of substantially more
late-stage covariance directions at Layer~4. Global average pooling then
aggregates these diverse spatial responses into a representation that is both
broader spectrally and markedly more symmetric and less heavy-tailed than the
BN+ReLU pooled representation. The classifier may therefore receive a richer
set of feature directions without an equally heavy-tailed pooled input. This
combination provides a credible mechanism for improved discrimination: SALU
need not impose batch-derived centering and scaling, provided that useful
feature diversity is preserved and the final aggregated representation remains
well behaved.

This diagnostic compares one checkpoint from each architecture. Repeating the
spatial sampling confirms numerical stability of the covariance estimates,
but it does not replace independent training runs. The observed spectral and
moment differences therefore provide a testable explanation rather than a
causal proof. The reproducibility of this combination of late-stage spectral
breadth and regularized pooled moments across independently trained checkpoints
remains to be established.

BN standardizes batch means and variances before its learned affine transformation; it does not guarantee Gaussian outputs, controlled higher moments, or unit variance after subsequent activations.

%% file: appendix_experimental_details.tex
\section{Detailed Experimental Protocols}
\label{app:experimental_details}

\subsection{CIFAR Protocols and Additional Comparisons}
\label{app:cifar_protocols}

For CIFAR-100, SaluNet-C-18 and its ResNet-18 comparisons use
a 300-epoch recipe with batch size 512. Training uses SGD with learning rate $0.3$, momentum $0.9$,
Nesterov acceleration, weight decay $5\times10^{-4}$, five warmup epochs, and
a cosine schedule with minimum learning rate $10^{-4}$. We use mixed
precision, gradient clipping at $1.0$, and dropout $0.3$. The augmentation
pipeline combines CIFAR AutoAugment, Cutout of length 16 with probability
$0.5$, label smoothing $0.05$, MixUp $\alpha=0.2$, and CutMix $\alpha=1.0$;
the latter two are applied through the same progressive mixing policy with
probability $0.5$. EMA decay is 0.997 for BN configurations and 0.9997 for the other
stabilizers.

For the CIFAR-10 SALU+GALU configuration, training also uses
300 epochs and batch size 512, but label smoothing is 0.01,
MixUp $\alpha=0.1$, and the mixing probability is 0.7.
CutMix retains $\alpha=1.0$.
This configuration uses a drop-path rate of 0.005 and EMA decay 0.999.

SALU is initialized at $(a,b)=(1,0.1)$ and uses exponential positivity,
whereas the SWALU and GALU gates are initialized at $(1,1)$ and use the
absolute-value parameterization employed by the CIFAR implementation. All learnable
signal-control and gate parameters have zero weight decay, use twice the main
learning rate, and follow a cosine schedule accelerated by a factor of $1.2$;
their learning rate is set to zero after its first minimum. This rule is
inactive for BN+ReLU because it contains no geometric parameters. GroupNorm
uses eight groups, and RMSNorm uses a learned affine scale.

SALU stabilizers use exponential positivity. SWALU and GALU use
effective positive parameters
$a=|a_{\mathrm{raw}}|+\varepsilon_p$ and
$b=|b_{\mathrm{raw}}|+\varepsilon_p$, with
$\varepsilon_p=10^{-6}$.

Both SWALU and GALU are evaluated on each CIFAR dataset, with three
independent runs per configuration. Updated aggregate CNN results are
given in Table~\ref{tab:cifar_cnn_results}. The additional comparisons
reported below are retained from the previous experiment set and are
not used to compute the updated SaluNet means.

\begin{table}[H]
\centering
\caption{Additional ResNet-18/CIFAR-100 comparisons retained from the original experiment set. Values are mean $\pm$ sample standard deviation over three independent runs; Best is the retained best EMA record.}
\label{tab:app_additional_cnn}
\small
\begin{tabular}{@{}lccc@{}}
\toprule
Stabilizer + activation & EMA (\%) & Main (\%) & Best (\%) \\
\midrule
RMSNorm + ReLU & $80.54\pm0.31$ & $79.71\pm0.22$ & 80.87 \\
RMSNorm + SWALU & $82.39\pm0.30$ & $81.62\pm0.23$ & 82.69 \\
DyT + ReLU & $79.93\pm0.16$ & $79.12\pm0.03$ & 80.08 \\
DyT + SWALU & $81.42\pm0.15$ & $79.30\pm0.06$ & 81.58 \\
\bottomrule
\end{tabular}
\end{table}

Table~\ref{tab:app_additional_cnn} provides additional CNN comparisons with
RMSNorm and DyT. Replacing ReLU by SWALU increases the observed EMA mean
by $1.85$ points with RMSNorm and $1.49$ points with DyT; the corresponding
main-model gains are $1.91$ and $0.18$ points. Under SWALU, SALU reaches
$83.02\pm0.29\%$ with EMA (Table~\ref{tab:cifar_cnn_results}), compared
with $82.39\pm0.30\%$ for RMSNorm and $81.42\pm0.15\%$ for DyT.
The corresponding main-model means are $82.36\pm0.31\%$ for
SALU+SWALU, $81.62\pm0.23\%$ for RMSNorm+SWALU, and
$79.30\pm0.06\%$ for DyT+SWALU.
These results show that the effect of the gate depends on the stabilizer;
they do not establish a general ranking across architectures.

\paragraph{Compact CIFAR Transformer.}
All baselines and SaluNet variants are trained in our implementation for 200
epochs with three independent runs, using both main and EMA evaluation.
The compact ViT recipe uses Adam, batch size 128, weight decay
$5\times10^{-5}$ for ordinary weights, and EMA decay $0.9999$.
Geometric parameters use one tenth of the ordinary learning rate and zero
weight decay. Five warmup epochs precede cosine decay, with learning-rate
minima $10^{-5}$ for ordinary weights and $10^{-6}$ for geometric parameters.
The pipeline uses CIFAR AutoAugment, random cropping with four-pixel constant
padding, and label smoothing $0.1$, without MixUp, CutMix, or Cutout.
The learning rate is $5\times10^{-4}$ for CIFAR-10: $10^{-3}$ diverged in
our trials. CIFAR-100 uses $10^{-3}$. For both datasets, SALU stabilizers
start at $(a,b)=(1,0.1)$ and SWALU/GALU gates at $(1,1)$.
DyT uses $\alpha_{\rm init}=1$ rather than $0.5$, because the former yielded
higher accuracies in our trials. This is an empirically selected
initialization, not an unmodified default comparison.
The compact architecture and attention scaling are specified in
Appendix~\ref{app:transformer_architectures}.

\subsection{Batch-Size Robustness}
\label{app:batch_size_robustness}

We evaluate the sensitivity of ResNet-18 training to the physical minibatch
size using the same CIFAR data pipeline and architecture family as in the
main comparison. Batch size is varied over
$\{512,256,128,64,32,16,8,4,2,1\}$.
For both CIFAR-10 and CIFAR-100, the initial learning rate scales linearly with physical batch size $B$, following $\eta(B)=0.3B/512$. It is therefore halved
at every halving of batch size, from $0.3$ at $B=512$ to
$0.0005859375$ at $B=1$. This rule is applied identically to
BN+ReLU and SaluNet.
SaluNet uses GALU for CIFAR-10 and SWALU for CIFAR-100.
Within each dataset and batch size, the compared models share the recipe
except for the stated operators.

Table~\ref{tab:app_batch_size_robustness} reports the retained
individual-run accuracies. On CIFAR-100, BN+ReLU remains competitive for moderate batch
sizes but deteriorates rapidly below $32$. In contrast, SaluNet-C-18 degrades
gradually and remains trainable at batch size $1$. The CIFAR-10 SaluNet sweep
shows the same qualitative resilience, although its BN+ReLU trajectory is not
included in the retained table.

\begin{table}[h]
\centering
\caption{Top-1 accuracy (\%) across physical training batch sizes.
Configurations with batch size at least $8$ were explored across multiple
seeds; batch sizes $4$, $2$, and $1$ use one completed run per configuration.
The table reports retained individual-run accuracies, not averages across
seeds; run-to-run standard deviations are therefore not reported.
C10 and C100 denote CIFAR-10 and CIFAR-100.}
\label{tab:app_batch_size_robustness}
\small
\setlength{\tabcolsep}{3.5pt}
\begin{tabular}{@{}lrrrrrrrrrr@{}}
\toprule
Batch size & 512 & 256 & 128 & 64 & 32 & 16 & 8 & 4 & 2 & 1 \\
\midrule
BN+ReLU (C100)
& 82.15 & 82.32 & 82.55 & 81.90 & 81.27
& 78.49 & 67.90 & 46.47 & 30.90 & 1.10 \\
SaluNet-C-18 (C100)
& 83.25 & 82.86 & 82.00 & 81.74 & 81.38
& 80.91 & 80.69 & 80.07 & 79.12 & 76.23 \\
SaluNet-C-18 (C10)
& 97.35 & 97.25 & 97.30 & 97.20 & 97.35
& 97.10 & 97.08 & 96.60 & 95.57 & 93.44 \\
\bottomrule
\end{tabular}
\end{table}

Batch sizes $4$, $2$, and $1$ are represented by one completed run per
configuration because of their substantially greater training cost.
These measurements document a pronounced batch-size robustness trend,
but do not quantify run-to-run variability. A future controlled study could compare SALU,
BN, GN, and RMSNorm under the same SWALU gate to separate dependence on
minibatch statistics from the contribution of the learned gate itself.

\subsection{Reporting and Evaluation Conventions}
\label{app:reporting_conventions}

We report Top-1 classification accuracy on CIFAR-10 and CIFAR-100, and Top-1
and Top-5 accuracy on the ImageNet-1K validation set. All accuracies are
percentages; CIFAR results use the standard test sets.

\paragraph{Main and EMA models.}
SaluNet-C-18, SaluNet-C-50, SaluNet-T-CIFAR, and SaluNet-T-B/16 maintain both the directly
optimized main model and an exponential moving average (EMA) of its parameters.
Unless a table explicitly identifies a main-model result, reported values for
these architectures refer to EMA. CIFAR tables explicitly separate both variants.

For the CNN CIFAR-100 comparison, BN-based models use EMA decay $0.997$, whereas
SALU-, GN-, and RMSNorm-based models use $0.9997$. A common decay of $0.9997$
caused the BN EMA to deteriorate, while $0.997$ produced its strongest stable
validation behavior. Because EMA does not alter the directly optimized model,
the main-model column provides the comparison independently of this smoothing
choice; EMA results are additionally reported using the stable coefficient of
each normalization family.

\paragraph{Checkpoint selection.}
Following the evaluation convention of the CIFAR training pipeline, the
official test set is evaluated after every epoch. Best main-model and EMA
checkpoints are selected independently from their respective test-accuracy
trajectories. The main-model accuracy stored in an EMA-selected checkpoint is
not interpreted as the independently best main-model result. When both maxima
are reported, each is selected from its own trajectory. The same independent selection rule is used for the compact CIFAR ViT.
Repeated test-set selection can introduce optimism; these values are not
held-out estimates following validation-only checkpoint selection.

\paragraph{Repeated runs.}
For $R$ independent runs we report
\begin{equation}
\overline A=\frac1R\sum_{r=1}^R A_r,\qquad
s_A=\sqrt{\frac1{R-1}\sum_{r=1}^R(A_r-\overline A)^2}.
\label{eq:accuracy_mean_std}
\end{equation}
The repository provides reference configurations, selected checkpoints, and available training logs. Complete per-run records, including all seed identifiers and individual results underlying the reported aggregates, are not included in the current release. Unless stated otherwise, controlled ResNet-18/CIFAR comparisons use three runs.
Single-run results are reported without uncertainty and are interpreted as evidence of applicability rather than estimates of run-to-run variability.

\paragraph{Comparison sources.}
A \emph{matched baseline} is trained in our environment with the same
backbone, data pipeline, primary optimization schedule, augmentation, epoch
budget, and evaluation protocol as its SaluNet counterpart, except for the
operators being compared. Operator-specific parameter groups and EMA
coefficients are disclosed separately. Values taken
from a public implementation or publication are marked as \emph{reported} and
are not presented as controlled reruns.

\paragraph{ImageNet evaluation resolution.}
Unless stated otherwise, ImageNet-1K uses a $224\times224$ center crop. An
additional resolution denotes reevaluation of the same checkpoint without
further training or fine-tuning.

\subsection{ImageNet-1K Protocols}
\label{app:imagenet_protocols}

Before training, the ImageNet-1K training images for both models are resized offline so that their longest side is 320 pixels, preserving aspect ratio with Lanczos interpolation, and saved as JPEG images at quality 95 using Pillow's default chroma subsampling. Online training augmentations operate on this preprocessed copy.
The final SaluNet-C-50 accuracies reported below are independently
recomputed on the original 50,000-image ImageNet-1K validation set.

\paragraph{SaluNet-C-50.}
The C-50 sampler repeats every training index four times and shuffles the
complete list, with \texttt{RA\_REPS}$=4$ and no truncation. For a dataset
of $N$ images, each scheduled epoch processes $4N$ views with independently
sampled online augmentations; 150 epochs therefore process $600N$ views.
Grouping four presentations per image within an epoch retains this exposure
budget while requiring 150 validation rounds for both the main and EMA
models, rather than 600 with one presentation per epoch. This reduces
intermediate evaluation overhead, not training computation at equal exposure.
The ten-epoch warmup spans $40N$ views. EMA is updated at optimizer steps;
changing the epoch grouping alone does not accelerate the adaptation of
$a,b$. An RA=1 run would need 600 epochs and a 40-epoch warmup to approximately
match the learning-rate schedule in processed views, but example ordering,
validation frequency, and incomplete accumulation groups at epoch boundaries
would still differ. Unlike our sampler, TorchVision's repeated-augmentation
sampler truncates its output to approximately $N$ views per epoch across
workers~\citep{torchvisionrasampler}. Repetition factors alone therefore do
not determine comparable training budgets.

SaluNet-C-50 trains from scratch for 150 epochs using
$224\times224$ random crops. The retained run uses physical batch size 180
and five-step gradient
accumulation, giving an effective batch size of 900 on one GPU. SGD uses
momentum $0.9$, base learning rate $0.4\times900/1024=0.3515625$, weight decay
$2\times10^{-5}$ for ordinary parameters, and zero weight decay for SALU and
SWALU parameters. The latter use twice the instantaneous main learning rate.
Both SALU and SWALU use exponential positivity in this run.
As a precaution against extreme activation values in this initial
large-scale CNN run, the SWALU output is clamped to $[-10,10]$.
All reported SaluNet-C-50 ImageNet results correspond to this
clamped variant; the effect of removing the clamp has not been evaluated.
Training applies ten warmup epochs followed by cosine decay, EMA with update-dependent decay
$d_t=\min\!\left(0.9999,\frac{1+t}{10+t}\right)$,
where $t$ counts EMA updates starting at $1$, label smoothing $0.1$, MixUp with $\alpha=0.2$, CutMix with
$\alpha=0.2$, random erasing with probability
$0.5$, mixed precision, and gradient clipping at $0.5$.

MixUp and CutMix are selected with equal probability. In both branches,
the initial mixing coefficient is sampled from $\mathrm{Beta}(0.2,0.2)$;
for CutMix, the label coefficient is then recomputed from the actual replaced
area after boundary clipping.

For final validation at $224^2$, each original validation image is resized
to 240 pixels along its shorter edge using bilinear interpolation and
then center-cropped to $224\times224$.
The $288^2$ result loads the same selected EMA checkpoint, resizes the
original image to 320 pixels along its shorter edge using bilinear
interpolation, and center-crops it to $288\times288$.
No additional training or fine-tuning is performed.
The checkpoint obtains $79.59/94.59\%$ Top-1/Top-5 at $224^2$
and $80.16/94.96\%$ at $288^2$.

\paragraph{SaluNet-T-B/16.}
We train SaluNet-T-B/16 from scratch for 300 epochs in the public Derf
ViT-B/16 framework~\citep{chen2025derfrepository}, replacing LayerNorm by
SALU and GELU by SWALU. We adapt the Derf training recipe to an effective batch size of $128$.
Training initially uses a peak learning rate of $1.25\times10^{-4}$
with a 20-epoch warmup. After epoch 150 (zero-based indexing), training
resumes from the corresponding checkpoint with the learning-rate
parameter set to $10^{-4}$ and no additional warmup, retaining the
300-epoch training horizon.
The completed run uses physical batch size $128$, no gradient accumulation,
and EMA decay $0.999$. The remaining augmentation,
regularization, and scheduling choices are inherited from that framework.
SALU uses exponential positivity through learned log-parameters.
SWALU uses effective coefficients $a=|\tilde{a}|+\varepsilon$ and $b=|\tilde{b}|+\varepsilon$,
with $\varepsilon=10^{-6}$.

The run is performed on one shared NVIDIA A40 GPU with $24$~GB of memory allocated to the training process. With physical batch size $128$ and no gradient accumulation on one GPU,
the effective batch size is also $128$.
For final validation, the 50,000 original ImageNet-1K validation images
are first resized offline to a longest-side length of 320 pixels,
preserving aspect ratio with Lanczos interpolation, and saved as JPEG images at quality 95 using Pillow's default
chroma subsampling. At $224^2$, these preprocessed images are resized online to 256 pixels
along their shorter edge using bicubic interpolation and then
center-cropped to $224\times224$.
At $288^2$, the online resize uses a shorter-side length of 320 pixels,
followed by a $288\times288$ center crop.
The learned positional embedding is interpolated bicubically from its
$14\times14$ training grid to the $18\times18$ evaluation grid.
Both evaluations use the same selected EMA checkpoint, without
additional training or fine-tuning.
The checkpoint reaches $80.12/94.79\%$ Top-1/Top-5 at $224^2$
and $80.28/94.74\%$ at $288^2$.
Because the run does not reproduce the much larger effective batch used by
the published Derf recipe, it is a scalability test rather than a
recipe-matched operator comparison.

\paragraph{Scope.}
Both ImageNet models are trained on a single shared NVIDIA A40 GPU with
$24$~GB of memory available to the training process. The C-50 run uses
five-step gradient accumulation; the T-B/16 run uses none. These are
architecture-specific training configurations, not an ablation of accumulation.
Each ImageNet backbone is represented by one completed run. We
therefore report no standard deviation and do not attribute small differences
from published baselines to SALU. Repeated ImageNet runs and recipe-matched
baselines remain desirable but are substantially more expensive than the
controlled CIFAR study.
The corresponding validation accuracies, training budgets, physical/effective
batch sizes, and evaluation resolutions are summarized once in
Table~\ref{tab:imagenet_scaling} of the main paper.

\paragraph{Execution mode and supplementary record.}
The primary experiments use \texttt{deterministic: false} and
\texttt{cudnn\_benchmark: true}.
Repeated comparisons use the same three seed identifiers across compared
configurations, rather than selecting the best three exploratory runs.
The SALU+SWALU/CIFAR-100 means $83.02\pm0.29$ (EMA) and
$82.36\pm0.31$ (main) summarize these three runs; their maximum EMA
accuracy is $83.27\%$.
An additional exploratory run with seed $142$ used
\texttt{deterministic: true} and \texttt{cudnn\_benchmark: false}.
Its EMA accuracy of $83.42\%$ is excluded from those means, standard
deviations, and the corresponding main-table maximum.
This additional run does not constitute a controlled comparison
of execution modes.

\paragraph{Terminal geometric-parameter freeze.} In configurations using \texttt{signal.after\_min: freeze\_zero}, the geometric-parameter learning rate becomes zero after its first scheduled minimum, while ordinary weights continue their own schedule. The onset depends on the configured decay factor and training duration, not a fixed number of final epochs. For example, \texttt{warmup\_epochs: 5}, \texttt{min\_lr: 0.0001}, and both cosine and signal \texttt{decay\_factor: 1.2} specify one such configuration. The repository also supports holding the minimum or allowing a subsequent rise. Preliminary trials favored the zero-after-minimum policy, but no quantitative three-policy ablation is reported. This policy does not describe the constant-rate 25-epoch diagnostics.

%% file: appendix_final_geometry_values.tex
\section{Final Logged Geometry: Numerical Values and Bar Charts}
\label{app:final_geometry_values}
The tables report the effective coordinates used to generate the function
plots in Figures~\ref{fig:learned_salu_geometry}
and~\ref{fig:learned_swalu_geometry}, at the end of training of the supplied run.
For SALU, the logged positive coordinates are used directly. For SWALU,
the effective coordinates are $|a_{\mathrm{raw}}|+10^{-6}$ and
$|b_{\mathrm{raw}}|+10^{-6}$. Values are rounded for display. The bar charts retain the geometric
summary of the original figures. These logged values are not presented as
an extraction from the best EMA checkpoint at an earlier training epoch.

\begin{table}[htbp]
\centering
\caption{Effective SALU parameters at the end of training.}
\begin{tabular}{lrr}
\toprule
Module & $a$ & $b$ \\
\midrule
Stem & 0.708388 & 1.731663 \\
L1.0-SALU1 & 15.06458 & 18.25309 \\
L1.0-SALU2 & 0.8820178 & 9.729508 \\
L1.1-SALU1 & 3.982746 & 325.9538 \\
L1.1-SALU2 & 1.46425 & 0.06869393 \\
L2.0-SALU1 & 3.174119 & 1.912205 \\
L2.0-SALU2 & 2.104288 & 27.70615 \\
L2.0-Proj & 5.475471 & 138.0326 \\
L2.1-SALU1 & 1.850934 & 1.53638 \\
L2.1-SALU2 & 1.798447 & 0.2662487 \\
L3.0-SALU1 & 1.711117 & 1.518941 \\
L3.0-SALU2 & 1.631901 & 2.020442 \\
L3.0-Proj & 1.803766 & 60.71511 \\
L3.1-SALU1 & 1.303135 & 0.1595751 \\
L3.1-SALU2 & 2.660377 & 0.9465353 \\
L4.0-SALU1 & 6.217439 & 0.03234436 \\
L4.0-SALU2 & 45.86884 & 0.1162041 \\
L4.0-Proj & 126.706 & 2.876008 \\
L4.1-SALU1 & 7.79875 & 0.0009234765 \\
L4.1-SALU2 & 6.676677 & 1.603122 \\
\bottomrule
\end{tabular}
\end{table}

\begin{table}[htbp]
\centering
\caption{Effective SWALU parameters at the end of training.}
\begin{tabular}{lrr}
\toprule
Module & $a$ & $b$ \\
\midrule
Stem-SWALU & 0.006217514 & 2.459094 \\
L1.0-SWALU & 0.5877033 & 5.743887 \\
L1.1-SWALU & 2.916163 & $1.67893\times10^{-6}$ \\
L2.0-SWALU & 2.598337 & 0.001107186 \\
L2.1-SWALU & 4.089253 & 1.535922 \\
L3.0-SWALU & 3.657944 & 0.09813965 \\
L3.1-SWALU & 6.319636 & 0.9819916 \\
L4.0-SWALU & 6.456943 & 6.603198 \\
L4.1-SWALU & 0.1392721 & 0.02488796 \\
\bottomrule
\end{tabular}
\end{table}

\begin{figure}[!htbp]
\centering
\includegraphics[width=.92\linewidth]{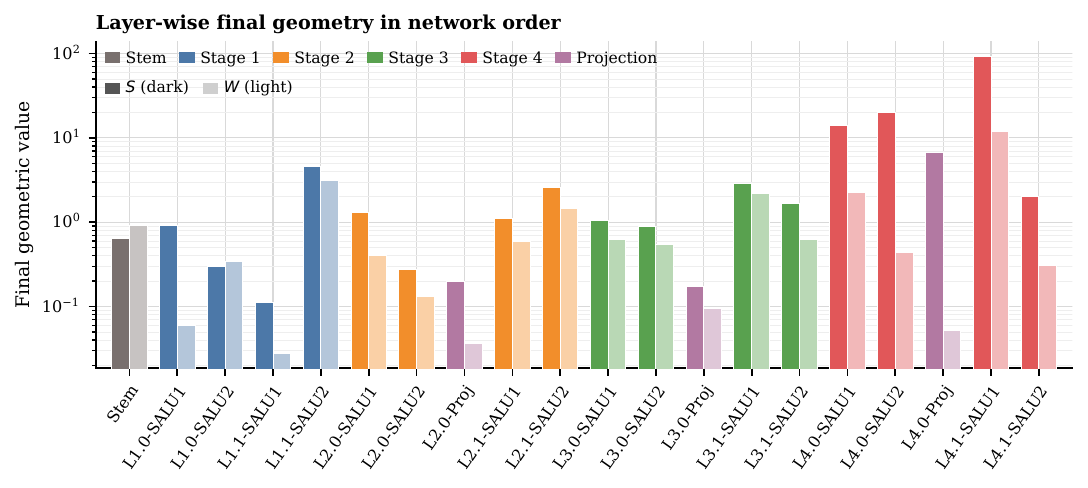}
\caption{Supplementary SALU geometric summary. Final amplitude $S$ (dark bars)
and width $W$ (light bars) for each module at the end of training.}
\end{figure}
\begin{figure}[!htbp]
\centering
\includegraphics[width=.92\linewidth]{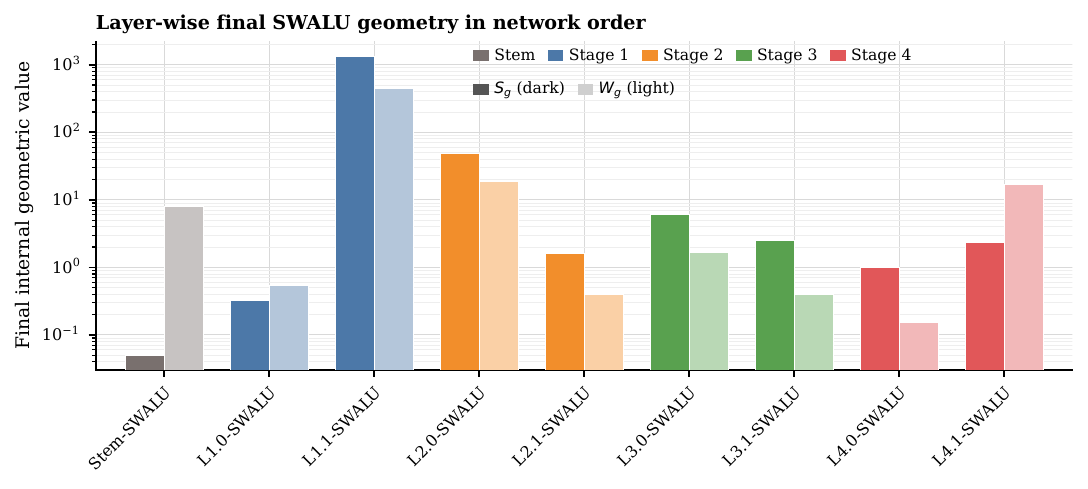}
\caption{Supplementary SWALU internal geometric summary. Final internal amplitude $S_g$
(dark bars) and width $W_g$ (light bars) at the end of training.}
\end{figure}

%% file: appendix_external_imagenet_context.tex
\section{Published ImageNet Reference Points and Comparison Limits}
\label{app:external_imagenet_context}

Table~\ref{tab:external_imagenet_context} summarizes selected published results
without activation normalization. It provides context for
Section~\ref{subsec:imagenet_scaling}, not a ranking under a common protocol.
The timm checkpoint and the NFNet paper results have different provenance;
training details from one must not be assigned to the other.

\begin{table}[h]
\centering
\small
\setlength{\tabcolsep}{4pt}
\caption{External ImageNet-1K reference points. NR denotes information not
confirmed in the consulted model documentation. NFNet-F0 reports a mean over
three seeds without a standard deviation in its main table; the other entries
are published point values without an across-run standard deviation in the
cited result tables or checkpoint documentation. These are not matched
comparisons with SaluNet.}
\label{tab:external_imagenet_context}
\begin{tabular}{@{}lrrrr@{}}
\toprule
Model / source & Parameters (M) & Epochs & Effective batch & Top-1 (\%) \\
\midrule
NF-ResNet-50 (timm)~\citep{wightman_nfresnet50} & 25.6 & NR & NR & 80.658 \\
NFNet-F0~\citep{brock2021high} & 71.5 & 360 & 4096 & 83.6 \\
ViT-B + DyT~\citep{chen2025strongernormalizationfreetransformers} & NR & 300 & 4096 & 82.5 \\
ViT-B + Derf~\citep{chen2025strongernormalizationfreetransformers} & NR & 300 & 4096 & 82.8 \\
\bottomrule
\end{tabular}
\end{table}

\paragraph{CNN architecture and resolution.}
The timm NF-ResNet-50 checkpoint has 25.6 million parameters, uses
pre-activation residual blocks, scaled weight standardization, and
signal-propagation-based scalar gains. Its training resolution is $256^2$ and
its reported $80.658\%$ is evaluated at $288^2$. SaluNet-C-50 instead retains
the torchvision ResNet-50 backbone, replaces BN and ReLU by SALU and SWALU,
and trains at $224^2$. Its $80.16\%$ evaluation at $288^2$ is lower by
$0.498$ percentage points. Neither result provides an across-run standard
deviation for this comparison, so the gap cannot be identified as sampling
variability or interpreted as statistical equivalence. Architecture,
preprocessing, and optimization differences also prevent attributing it
specifically to the stabilization operator. We do not infer the timm run's
epoch count from the NFNet paper.

NFNet-F0 is a larger, architecturally redesigned CNN with 71.5 million
parameters. It combines scaled weight standardization, residual scaling,
adaptive gradient clipping, and an SE-ResNeXt-derived backbone. Its
$83.6\%$ result uses training at $192^2$ and evaluation at $256^2$, with
360 epochs and effective batch size 4096~\citep{brock2021high}. This is an
ImageNet-only result; the paper's separate transfer-learning experiments
pretrain on 300 million labeled images and must not be mixed with it.

\paragraph{Transformer recipe.}
The cited DyT and Derf scores correspond to the ViT-B size class, rather
than a substantially larger ViT-L. In the Derf study's comparison,
the recipe uses AdamW, 300 epochs, effective batch size 4096, base learning
rate $4\times10^{-3}$, 20 warmup epochs, and EMA decay $0.9999$
~\citep{chen2025strongernormalizationfreetransformers}. These are supervised
ImageNet-1K training results, not reported as transfer from an external
pretraining dataset. Our SaluNet-T-B/16 experiment adapts the Derf codebase to an effective batch size of 128 and EMA decay 0.999, and also replaces GELU with SWALU.
Its two-stage learning-rate configuration is detailed in Appendix~\ref{app:imagenet_protocols}. The resulting Top-1 accuracies of $80.12\%$ at $224^2$ and $80.28\%$ at $288^2$ therefore establish feasibility of this configuration, not superiority over DyT or Derf. A larger batch does not by itself imply more processed examples or a larger total compute budget.

\paragraph{Scope of the contribution.}
SaluNet's contribution is a common learned pointwise geometry for stabilization
and adaptive gating, applied to BN-based ResNets and LN-based vision
Transformers. The controlled CIFAR results and the CNN batch-size sweep support
specific empirical claims; they do not establish a global accuracy record or
batch-size robustness for all Transformer configurations. Published ImageNet
reference points do not replace operator comparisons trained under a common
recipe. Repeated ImageNet runs, matched normalized and normalization-free
baselines, and Transformer batch-size sweeps remain future work.